%% file: main.tex
\documentclass{article}

\PassOptionsToPackage{numbers, compress}{natbib}

\usepackage[final]{neurips_2023}

\usepackage{xcolor}   
 \definecolor{mydarkblue}{rgb}{0,0.08,0.45}
 \definecolor{LightGray}{gray}{0.9}
 \usepackage[colorlinks=true,
     linkcolor=mydarkblue,
     citecolor=mydarkblue,
     filecolor=mydarkblue,
     urlcolor=mydarkblue]{hyperref}
\usepackage{hyperref}

\usepackage[utf8]{inputenc} 
\usepackage[T1]{fontenc}    
\usepackage{hyperref}       
\usepackage{url}            
\usepackage{booktabs}       
\usepackage{amsfonts}       
\usepackage{nicefrac}       
\usepackage{microtype}      

\usepackage{amsmath}
\usepackage{amssymb}
\usepackage{mathtools}
\usepackage{amsthm}

\usepackage{hyperref}
\usepackage{url}

\usepackage{amsmath}
\usepackage{amssymb}
\usepackage{wrapfig}
\usepackage{multirow}
\usepackage{tabularx}
\usepackage{adjustbox}
\usepackage{color, colortbl}
\usepackage{float} 
\usepackage{subfigure}
\usepackage{makecell}
\usepackage{booktabs} 
\usepackage{amsfonts}
\usepackage{enumerate}
\usepackage{bm}
\usepackage{slashbox}
\usepackage{diagbox}
\usepackage[marginal]{footmisc}
\usepackage{sidecap}
\usepackage[english]{babel}

\newtheorem{theorem}{Theorem}[section]

\usepackage{color, colortbl}

\usepackage{tikz}
\usepackage{pifont}
\newcommand{\ymark}{\ding{51}}%
\newcommand{\xmark}{\ding{55}}%

\usepackage{algorithmic}
\usepackage[ruled,vlined]{algorithm2e}

\usepackage[frozencache=true,cachedir=minted-cache]{minted}

\title{Episodic Multi-Task Learning with \\
Heterogeneous Neural Processes}

\author{Jiayi Shen$^1$, Xiantong Zhen$^{1,2}$ ~\thanks{Currently with United Imaging Healthcare, Co., Ltd., China.}~, Qi (Cheems) Wang$^{3}$, Marcel Worring$^1$\\
$^1$University of Amsterdam, Netherlands, \texttt{\{j.shen, m.worring\}@uva.nl}\\
$^2$ Inception Institute of Artificial Intelligence, Abu Dhabi, UAE, \texttt{zhenxt@gmail.com} \\
$^3$ Kaiyuan Mathematical Sciences Institute, Changsha, China, \texttt{hhq123go@gmail.com} 
}

\begin{document}

\maketitle

\input{0_abstract}
\input{1_introduction}

\input{3_methodology}

\input{2_related}

\input{4_experiments}
\input{5_conclusion}
{
\small
\bibliography{main}
\bibliographystyle{unsrtnat}
}
\newpage
\input{6_appendix}
\end{document}

%% file: 0_abstract.tex
\begin{abstract}

This paper focuses on the data-insufficiency problem in multi-task learning within an episodic training setup. Specifically, we explore the potential of heterogeneous information across tasks and meta-knowledge among episodes to effectively tackle each task with limited data. Existing meta-learning methods often fail to take advantage of crucial heterogeneous information in a single episode, while multi-task learning models neglect reusing experience from earlier episodes. To address the problem of insufficient data, we develop Heterogeneous Neural Processes (HNPs) for the episodic multi-task setup. Within the framework of hierarchical Bayes, HNPs effectively capitalize on prior experiences as meta-knowledge and capture task-relatedness among heterogeneous tasks, mitigating data-insufficiency. Meanwhile, transformer-structured inference modules are designed to enable efficient inferences toward meta-knowledge and task-relatedness. In this way, HNPs can learn more powerful functional priors for adapting to novel heterogeneous tasks in each meta-test episode. Experimental results show the superior performance of the proposed HNPs over typical baselines, and ablation studies verify the effectiveness of the designed inference modules.
\end{abstract}

%% file: 1_introduction.tex
\section{Introduction}
\label{sec: introduction}

Deep learning models have made remarkable progress with the help of the exponential increase in the amount of available training data~\cite{Goodfellow-et-al-2016}. 
However, many practical scenarios only have access to limited labeled data~\cite{xu2021weak}. Such data-insufficiency sharply degrades the model's performance~\cite{xu2021weak, johnson2019survey}.
Both meta-learning and multi-task learning have the potential to alleviate the data-insufficiency issue.  
Meta-learning can extract meta-knowledge from past episodes and thus enables rapid adaptation to new episodes with a few examples only~\cite{vinyals2016matching, finn2017model, snell2017prototypical, requeima2019fast}. 
Meanwhile, multi-task learning exploits the correlation among several tasks and results in more accurate learners for all tasks simultaneously~\cite{caruana1997multitask, zhang2021survey, long2017learning, shen2021variational}.
However, the integration of meta-learning and multi-task learning in overcoming the data-insufficiency problem is rarely investigated.

In episodic training~\cite{vinyals2016matching}, existing meta-learning methods ~\cite{vinyals2016matching, finn2017model, snell2017prototypical, requeima2019fast, garnelo2018neural, garnelo2018conditional} in every meta-training or meta-test episode learn a single-task. In this paper, we refer to this conventional setting as \textit{episodic single-task learning}. This setting restricts the potential for these models to explore task-relatedness within each episode, leaving the learning of multiple heterogeneous tasks in a single episode under-explored. We consider multiple tasks in each episode as \textit{episodic multi-task learning}. The crux of episodic multi-task learning is to generalize the ability of exploring task-relatedness from meta-training to meta-test episodes.
The differences between episodic single-task learning and episodic multi-task learning are illustrated in Figure~\ref{fig: MMTL_setting}. 
To be specific, we restrict the scope of the problem setup to the case where tasks in each meta-training or meta-test episode are heterogeneous but also relate to each other by sharing the same target space.

\begin{figure*}[t]
\begin{center}
\includegraphics[width=1.0\textwidth]{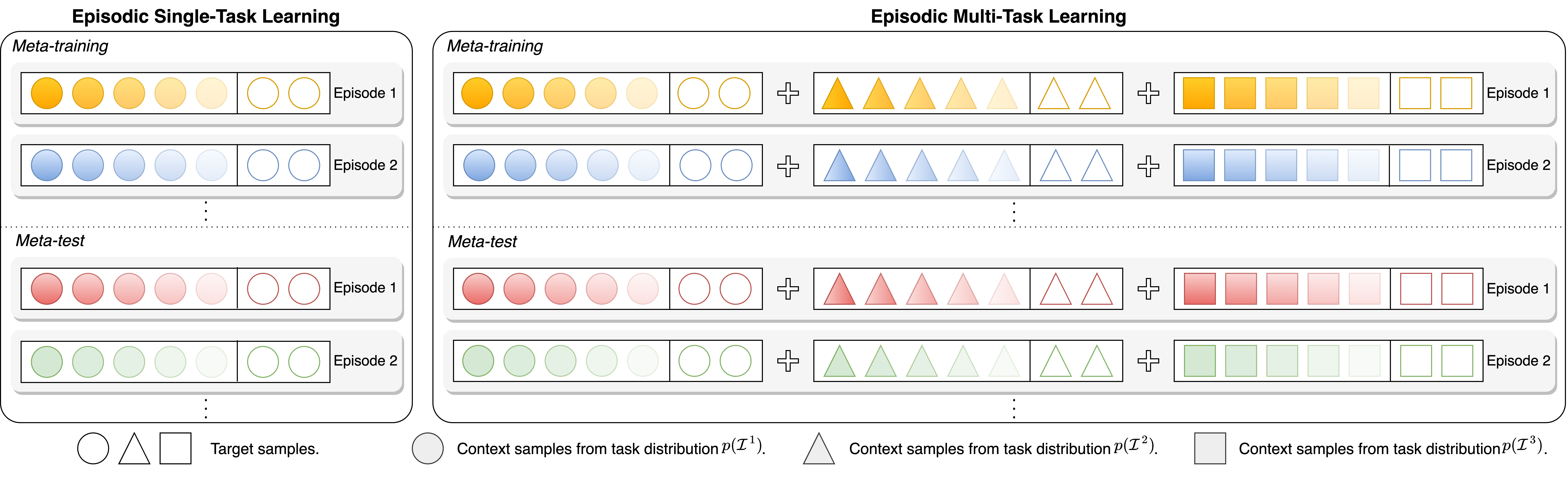}
\end{center}
\vspace{-2mm}
\caption{\textbf{Illustration of episodic multi-task learning.} 
Each row corresponds to a meta-training or meta-test episode.  
Different colors represent different label spaces among episodes; the same color with different shades represents different categories in the same task.
Compared with episodic single-task learning, episodic multi-task learning simultaneously handles several related tasks in a single episode.
}
\label{fig: MMTL_setting}
\vspace{-2mm}
\end{figure*}

The neural process (NP) family~\cite{garnelo2018neural, garnelo2018conditional}, as typical meta-learning probabilistic models~\cite{bruinsma2023autoregressive}, efficiently quantifies predictive uncertainty with limited data, making it in principle well-suited for tackling the problem of data-insufficiency. 
However, in practice, it is challenging for vanilla NPs~\cite{garnelo2018neural} with a global latent variable 
to encode beneficial heterogeneous information in each episode.
This issue is also known as the expressiveness bottleneck~\cite{kim2019attentive, wang2022learning}, which weakens the model's capacity to handle insufficient data, especially when faced with diverse heterogeneous tasks.

To better resolve the data-insufficiency problem, we develop Heterogeneous Neural Processes (HNPs) for episodic multi-task learning. As a new member of the NP family, HNPs improve the expressiveness of vanilla NPs by introducing a hierarchical functional space with global and local latent variables.
The remainder of this work is structured as follows: We introduce our method in Section (\ref{sec: method}).
Related work is overviewed in Section (\ref{sec: related work}). We report experimental results with analysis in Section (\ref{sec: experiment}), after which we conclude with a technical discussion, existing limitations, and future extensions. In detail, our technical contributions are two-fold:

\begin{itemize}
    \item Built on the hierarchical Bayes framework, our developed HNPs can simultaneously generalize meta-knowledge from past episodes to new episodes and exploit task-relatedness across heterogeneous tasks in every single episode. This mechanism makes HNPs more powerful when encoding complex conditions into functional priors.
    \item We design transformer-structured inference modules to infer the hierarchical latent variables, capture task-relatedness, and learn a set of tokens as meta-knowledge. The designed modules can fuse the meta-knowledge and heterogeneous information from context samples in a unified manner, boosting the generalization capability of HNPs across tasks and episodes. 
\end{itemize}

Experimental results show that the proposed HNPs together with transformer-structured inference modules, can exhibit superior performance on regression and classification tasks under the episodic multi-task setup.

%% file: 3_methodology.tex
\section{Methodology}
\label{sec: method}

\noindent\textbf{Notations~\footnote{For ease of presentation, we abbreviate a set $\{ {(\cdot)}^m \}^M_{m=1}$ as ${(\cdot)}^{1:M}$, where $M$ is a positive integer. Likewise, $\{ {(\cdot)}_o \}^O_{o=1}$ is abbreviated as ${(\cdot)}_{1:O}$. For convenience, the notation table is provided in Appendix B.}.} 
We will now formally define episodic multi-task learning. For a single episode $\tau$, we consider $M$ heterogeneous but related tasks $\mathcal{I}^{1:M}_{\tau} = \{\mathcal{I}^{m}_{\tau}\}_{m=1}^{M}$. Notably, the subscript denotes an episode, while superscripts are used to distinguish tasks in this episode. In the episodic multi-task setup, tasks in a single episode are heterogeneous since they are sampled from different task distributions $\{p(\mathcal{I}^{m})\}_{m=1}^{M}$, but are related at the same time as they share the target space $\mathcal{Y}_{\tau}$.

To clearly relate to the modeling of vanilla neural processes~\cite{garnelo2018neural}, this paper follows its nomenclature to define each task. Note that in vanilla neural processes~\textit{context} and \textit{target} are often respectively called \textit{support} and \textit{query} in conventional meta-learning~\cite{vinyals2016matching, finn2017model}.
Each task $\mathcal{I}^m_{\tau}$ contains a context set with limited training data $\mathcal{C}_{\tau}^{m} = \{{\bar{x}}_{\tau, i}^m, {\bar{y}}_{\tau, i}^m\}_{i=1}^{N_\mathcal{C}}$ and a target set $\mathcal{T}_{\tau}^{m} = \{{{x}}_{\tau, j}^m, {{y}}_{\tau, j}^m\}_{j=1}^{N_\mathcal{T}}$, where $N_\mathcal{C}$ and $N_\mathcal{T}$ are the numbers of context samples and target samples, respectively. 
$\bar{{x}}_{\tau, i}^m$ and ${x}_{\tau, j}^m$ represent features of context and target samples; 
while $\bar{{y}}_{\tau, i}^m, {y}_{\tau, j}^m \in \mathcal{Y}_{\tau}$ are their corresponding targets, where~$i=1, 2, ..., N_\mathcal{C}; j=1, 2, .., N_\mathcal{T}; m=1, 2, ..., M$. 
For simplicity, we denote the set of target samples and their corresponding ground-truths by
$\mathbf{x}_{\tau}^m = \{{x}_{\tau, j}^m\}_{j=1}^{N_\mathcal{T}}$, $\mathbf{y}_{\tau}^m = \{{y}_{\tau, j}^m\}_{j=1}^{N_\mathcal{T}}$.
For an episode $\tau$, episodic multi-task learning aims to perform simultaneously well on each corresponding target set $\mathcal{T}_{\tau}^{m}, m=1, 2.., M$, given the collection of context sets $\mathcal{C}_{\tau}^{1:M}$.

For classification, this paper follows the protocol of meta models~\cite{vinyals2016matching, finn2017model, ravi2017optimization}, such as $\texttt{$O$-way $K$-shot}$ setup, clearly suffering from the data-insufficiency problem. Thus, episodic multi-task classification can be cast as a \texttt{$M$-task $O$-way $K$-shot} supervised learning problem. An episode has $M$ related classification tasks, and each of them has a context set with $K$ different instances from each of the $O$ classes~\cite{finn2017model}. It is worth mentioning that the target spaces of meta-training episodes do not overlap with any categories in those of meta-test episodes. 

\subsection{Modeling and Inference of Heterogeneous Neural Processes}
\label{sec: method/modeling}

We now present the proposed heterogeneous neural process. The proposed model inherits the advantages of multi-task learning and meta-learning, which can exploit task-relatedness among heterogeneous tasks and extract meta-knowledge from previous episodes. Next, we characterize the generative process, clarify the modeling within the hierarchical Bayes framework, and derive the approximate evidence lower bound (ELBO) in optimization.

\paragraph{Generative Processes.} 
To get to our proposed method HNPs, we extend the distribution over a single function $p(f_\tau)$ as used in vanilla NPs to a joint distribution of multiple functions $p(f^{1:M}_\tau)$ for all heterogeneous tasks in a single episode $\tau$. In detail, the underlying multi-task function distribution $p(f^{1:M}_\tau)$ is inferred from a collection of context sets $\mathcal{C}^{1:M}_{\tau}$ and learnable meta-knowledge $\omega, \nu^{1:M}$. 
Note that $\omega$ represents the shared meta-knowledge for all tasks, and $\nu^{m}$ denotes the task-specific meta-knowledge corresponding to the task distribution $p(\mathcal{I}^m)$.
Hence, we can formulate the predictive distribution for every single episode as follows:
\begin{small}
\begin{equation}
\begin{aligned}
& p(\mathcal{T}_{\tau}^{1:M} | \mathcal{C}_{\tau}^{1:M};  \omega, \nu^{1:M}) = \int p(\mathbf{y}_{\tau}^{1:M}|\mathbf{x}_{\tau}^{1:M}, f_{\tau}^{1:M}) p(f_{\tau}^{1:M}|\mathcal{C}_{\tau}^{1:M}; \omega, \nu^{1:M}) df_{\tau}^{1:M},
\end{aligned}
\label{eq: lossfunction1}
\end{equation}
\end{small}where $p(f_{\tau}^{1:M}|\mathcal{C}_{\tau}^{1:M}; \omega, \nu^m)$ denotes the data-dependent functional prior for multiple tasks of the episode $\tau$.
The functional prior encodes context sets from all heterogeneous tasks and quantifies uncertainty in the functional space. Nevertheless, it is less optimal to characterize multi-task function generative processes with vanilla NPs, since the single latent variable limits the capacity of the latent space to specify the complicated functional priors. This expressiveness bottleneck in vanilla NPs is particularly severe for our episodic multi-task learning since each episode has diverse heterogeneous tasks with insufficient data.

\begin{wrapfigure}{t}{0.50\textwidth}
\centering
\vspace{-7mm}
\includegraphics[width=0.9\linewidth]{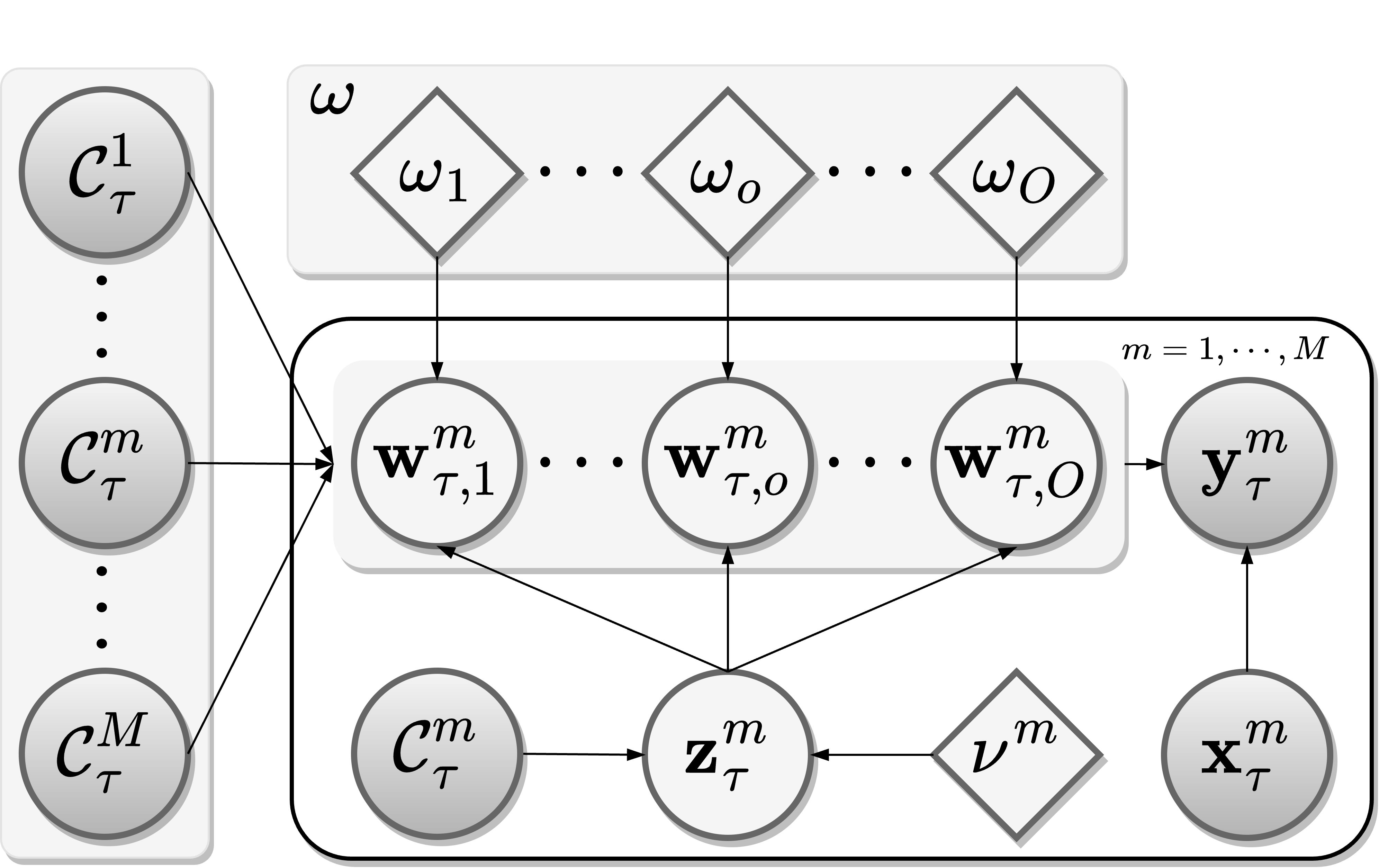}
\vspace{-2mm}
\caption{\textbf{Graphical model of the proposed HNPs in a single episode.} Filled shapes indicate observations. Probabilistic and deterministic variables are indicated by unfilled circles and diamonds, respectively.}
\label{fig: HNP_model}
\vspace{-5mm}
\end{wrapfigure}

\paragraph{Modeling within the Hierarchical Bayes Framework.}
To mitigate the expressiveness bottleneck of vanilla NPs, we model HNPs by parameterizing each task-specific function within a hierarchical Bayes framework. As illustrated in Figure~\ref{fig: HNP_model}, HNPs integrate a global latent representation $\mathbf{z}_{\tau}^m$ and a set of local latent parameters $\mathbf{w}_{\tau, 1:O}^m$ to model each task-specific function $f_{\tau}^{m}$.
Specifically, the latent variables are introduced at different levels: $\mathbf{z}_\tau^m$ encodes task-specific context information from $\mathcal{C}_{\tau}^{m}$ and $\nu^m$ in the representation level. $\mathbf{w}_{\tau, 1:O}^m$ encode prediction-aware information for a task-specific decoder from $\mathcal{C}_{\tau}^{1:M}$ and $\omega$ in the parameter level, where $O$ is the dimension of the decoder. For example, the dimension is the size of the target space when performing classification tasks. 

Notably, each local latent parameter is conditioned on the global latent representation, which controls access to all context sets in the episode for the corresponding task. Our method differs from previous hierarchical architectures~\cite{wang2022learning, wang2020doubly, kim2021multi, guo2023versatile} in the NP family since the local latent parameters of our HNPs are prediction-aware and explicitly constitute a decoder for the subsequent inference processes.

In practice, we assume that distributions of each task-specific function are conditionally independent. Thus, with the introduced hierarchical latent variables for each task in the episode, we can factorize the prior distribution over multiple functions in Eq. (\ref{eq: lossfunction1}) into:
\begin{equation}
\begin{aligned}
& p(f_{\tau}^{1:M}|\mathcal{C}_{\tau}^{1:M}; \omega, \nu^{1:M}) = \prod_{m=1}^M  p(\mathbf{z}_{\tau}^m|\mathcal{C}_{\tau}^{m}; \nu^m) p( \mathbf{w}_{\tau, 1:O}^m | \mathbf{z}_{\tau}^m, \mathcal{C}_{\tau}^{1:M}; \omega),
\end{aligned}
\label{eq: parameterization}
\end{equation}
where $p(\mathbf{z}_{\tau}^m|\mathcal{C}_{\tau}^{m};\nu^m)$ and $p( \mathbf{w}_{\tau, 1:O}^m | \mathbf{z}_{\tau}^m, \mathcal{C}_{\tau}^{1:M}; \omega)$ are prior distributions of the global latent representation and the local latent parameters to induce the task-specific function distribution. 

By integrating Eq. (\ref{eq: parameterization}) into Eq. (\ref{eq: lossfunction1}), we rewrite the modeling of HNPs in the following form:
\begin{small}
\begin{equation}
\begin{aligned}
p(\mathcal{T}_{\tau}^{1:M} | \mathcal{C}_{\tau}^{1:M};  \omega, \nu^{1:M})  =  & \prod_{m=1}^{M} \int  \Big \{ \int p(\mathbf{y}_{\tau}^{m}|\mathbf{x}_{\tau}^{m}, \mathbf{w}_{\tau, 1:O}^m) \\
& p( \mathbf{w}_{\tau, 1:O}^m | \mathbf{z}_{\tau}^m, \mathcal{C}_{\tau}^{1:M}; \omega) d \mathbf{w}^m_{\tau, 1:O}   \Big \} p(\mathbf{z}_{\tau}^m|\mathcal{C}_{\tau}^{m}; \nu^m) d\textbf{z}_{\tau}^m, 
\end{aligned}
\label{eq: HNP_model}
\end{equation}
\end{small}where $p(\mathbf{y}_{\tau}^{m}|\mathbf{x}_{\tau}^{m}, \mathbf{w}_{\tau, 1:O}^m)$ is the function distribution for the task $\mathcal{I}^m_{\tau}$ in HNPs. 
This distribution is obtained by the matrix multiplication of $\mathbf{x}_{\tau}^{m}$ and all local latent parameters $\mathbf{w}_{\tau, 1:O}^m$. 

Compared with most NP models ~\cite{garnelo2018neural, wang2022learning, wang2020doubly, kim2021multi} employing only latent representations, HNPs infer both latent representations and parameters in the hierarchical architecture from multiple heterogeneous context sets and learnable meta-knowledge.
Our model specifies a richer and more intricate functional space by leveraging the hierarchical uncertainty inherent in the context sets and meta-knowledge. This theoretically yields more powerful functional priors to induce multi-task function distributions.


Moreover, we claim that the developed model constitutes an \emph{exchangeable stochastic process} and demonstrate this via Kolmogorov Extension Theorem~\cite{klenke2013probability}. Please refer to Appendix B for the proof.

\paragraph{Approximate ELBO.}
Since both exact functional posteriors and priors are intractable, we apply variational inference to the proposed HNPs in Eq. (\ref{eq: HNP_model}). 
This results in the approximate ELBO:
\begin{small}
\begin{equation}
\begin{aligned}
& L_{\rm{HNPs}}(\omega, \nu^{1:M}, \theta, \phi) = \sum_{m=1}^{M} \bigg \{ \mathbb{E}_{q_{\theta}(\mathbf{z}_{\tau}^m|\mathcal{T}_{\tau}^m;\nu^m)}  \Big \{ \mathbb{E}_{q_{\phi}(\mathbf{w}^m_{\tau, 1:O} | \mathbf{z}_{\tau}^m, \mathcal{T}_{\tau}^{1:M}; \omega)}[\log p(\mathbf{y}_{\tau}^{m}|\mathbf{x}_{\tau}^{m}, \mathbf{w}^{m}_{\tau, 1:O})] \\
& - \mathbb{D}_{\rm{KL}}[q_{\phi}(\mathbf{w}^m_{\tau, 1:O} | \mathbf{z}_{\tau}^m, \mathcal{T}_{\tau}^{1:M}; \omega) || p_{\phi}(\mathbf{w}^m_{\tau, 1:O} | \mathbf{z}_{\tau}^m, \mathcal{C}_{\tau}^{1:M}; \omega)] \Big \} - \mathbb{D}_{\rm{KL}}[q_{\theta}(\mathbf{z}_{\tau}^m|\mathcal{T}_{\tau}^{m}; \nu^m)||p_{\theta}(\mathbf{z}_{\tau}^m|\mathcal{C}_{\tau}^{m}; \nu^m)]  \bigg \},
\end{aligned}
\label{eq: ELBO}
\end{equation}
\end{small}where $q_{\theta}(\mathbf{z}_{\tau}^m|\mathcal{T}_{\tau}^{m}; \nu^m)$ and $q_{\phi}(\mathbf{w}^m_{\tau, 1:O} | \mathbf{z}_{\tau}^m, \mathcal{T}_{\tau}^{1:M}; \omega)$ are variational posteriors of their corresponding latent variables. $\theta$ and $\phi$ are parameters of inference modules for $\mathbf{z}^m_{\tau}$ and $\mathbf{w}^m_{\tau, 1:O}$, respectively.
Following the protocol of vanilla NPs~\cite{garnelo2018neural}, the priors use the same inference modules as variational posteriors for tractable optimization. 
In this way, the KL-divergence terms in Eq. (\ref{eq: ELBO}) encourage all latent variables inferred from the context sets to stay close to those inferred from the target sets, enabling effective function generation with few examples.
Details on the derivation of the approximate ELBO and its tractable optimization are attached in Appendix C.

\begin{figure*}[t] 
\centering 
\centerline{\includegraphics[width=1.0\textwidth]{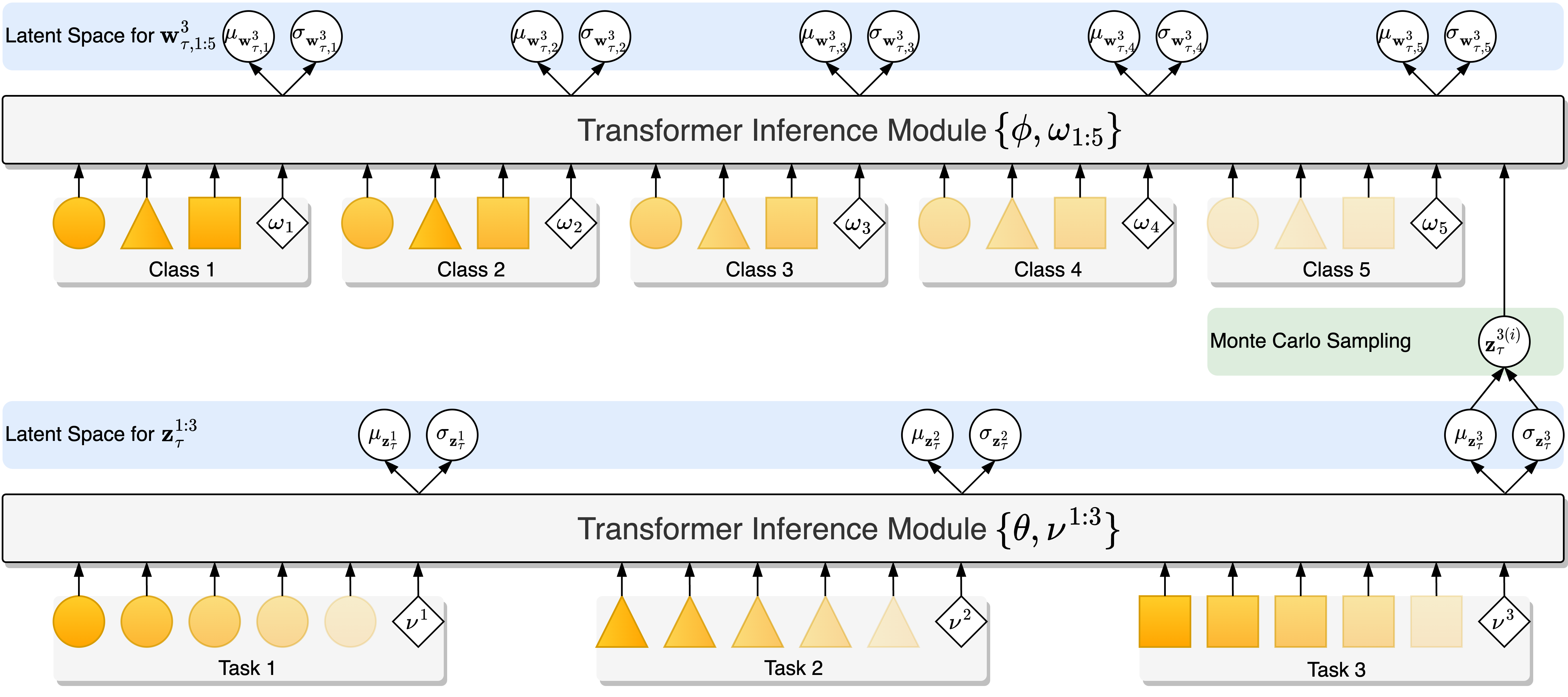}}
\caption{\textbf{A diagram of transformer-structured inference modules of HNPs for the first meta-training episode in Figure 1 under the \texttt{3-task 5-way 1-shot} setting.} For clarity, we display the inference process of the local latent parameters specific to the third task in the episode.}
\label{fig: diagram}
\end{figure*} 

\subsection{Transformer-Structured Inference Module}
\label{sec: method/transformer}

In order to infer the prior and variational posterior distributions in Eq.~(\ref{eq: ELBO}), it is essential to develop well-designed approximate inference modules. This is non-trivial and closely related to the performance of HNPs.
Here we adopt a transformer structure as the inference module to better exploit task-relatedness from the meta-knowledge and the context sets in the episode. More specifically, the previously mentioned meta-knowledge $\omega = \omega_{1:O}$ and $\nu^{1:M}$ are instantiated as learnable tokens to induce the distributions of hierarchical latent variables in the proposed model.

Without loss of generality, in the next, we provide an example of transformer-structured inference modules for prior distributions in classification scenarios.  Details of the inference modules in regression scenarios can be found in Appendix D. 
In Figure~\ref{fig: diagram}, a diagram of the transformer-structured inference modules is displayed under the \texttt{3-task 5-way 1-shot} setting.
In this case, the number of context samples is the same as the size of the target space, and thus we have $\mathcal{C}_{\tau}^m = \{\bar{x}^m_{\tau, o}, \bar{y}^m_{\tau, o}\}^{O}_{o=1}$, where $O$ is set as 5. In episodic training, labels in context sets are always available during inference. 

\paragraph{Transformer-Structured Inference Module $\{\theta, \nu^{m}\}$ for~$\mathbf{z}_{\tau}^{m}$.} 
In the proposed HNPs, each global latent representation encodes task-specific information relevant to the considered task in the episode as $p_{\theta}(\mathbf{z}^m_{\tau}|\mathcal{C}_{\tau}^{m}; \nu^m)$. The learnable token $\nu^m$ preserves the meta-knowledge from previous episodes for specific tasks, which are sampled from the corresponding task distribution $p(\mathcal{I}^m)$. 
The role of $\nu^m$ is to help the model adapt efficiently to such specific tasks in meta-test episodes. 

In detail, we set the dimension of the learnable token $\nu^m$ to the same as that of the features $\bar{x}^{m}_{\tau, 1:O}$. 
Then the transformer-structured inference module $\theta$ fuses them in a unified manner by taking $[\bar{x}^{m}_{\tau, 1:O}; \nu^m]$ as the input. 
The module $\theta$ outputs the mean and variance of the corresponding prior distribution. 
The inference steps for the global latent representation $\mathbf{z}^m_{\tau}$ are:
\begin{equation}
    [\widetilde{x}^{m}_{\tau, 1:O}; \widetilde{\nu}^m] = \texttt{MSA}( \texttt{LN} ([\bar{x}^{m}_{\tau, 1:O}; \nu^m])) + [\bar{x}^{m}_{\tau, 1:O}; \nu^m],
\label{eq: z_1}
\end{equation}
\begin{equation}
    [\widehat{x}^{m}_{\tau, 1:O}; \widehat{\nu}^m] = \texttt{MLP}( \texttt{LN} ([\widetilde{x}^{m}_{\tau, 1:O}; \widetilde{\nu}^m])) + [\widetilde{x}^{m}_{\tau, 1:O}; \widetilde{\nu}^m],
\label{eq: z_2}
\end{equation}
\begin{equation}
    p_{\theta}(\mathbf{z}^m_{\tau}|\mathcal{C}_{\tau}^{m}; \nu^m)  = \mathcal{N}(\mathbf{z}^m_{\tau} ; \mu_{\mathbf{z}^{m}_{\tau}}, \sigma_{\mathbf{z}^{m}_{\tau}}),
\label{eq: z_3}
\end{equation}
where~$\mu_{\mathbf{z}^{m}_{\tau}} = \texttt{MLP}(\widehat{\nu}^m), \sigma_{\mathbf{z}^{m}_{\tau}}=\texttt{Softplus}(\texttt{MLP}(\widehat{\nu}^m))$. The transformer-structured inference module includes a multi-headed self-attention ($\texttt{MSA}$) and three multi-layer perceptrons ($\texttt{MLP}$). The layer normalization ($\texttt{LN}$) is "pre-norm" as done in~\cite{kim2021vilt}.  \texttt{Softplus} is the activation function to output the appropriate value as the variance of the prior distribution~\cite{sonderby2016ladder}. 

\paragraph{Transformer-Structured Inference Module $\{\phi, \omega_{1:O}\}$ for $\mathbf{w}_{\tau, 1:O}^m$.} Likewise, each learnable token $\omega_{o}$ corresponds to a local latent parameter $\mathbf{w}_{\tau, o}^m$. 
With the learnable tokens $\omega_{1:O}$, we reformulate the prior distribution of local latent parameters as $p_{\phi}(\mathbf{w}^m_{\tau, 1:O} | \mathbf{z}_{\tau}^m, \mathcal{C}_{\tau}^{1:M}; \omega_{1:O})$. 
In this way, we learn the shared knowledge, inductive biases across all tasks, and their distribution at a parameter level, which in practical settings can capture epistemic uncertainty.

To be specific, the prior distribution can be factorized as $\prod_{o=1}^O p_{\phi}(\mathbf{w}^m_{\tau, o} | \mathbf{z}_{\tau}^m, \mathcal{C}_{\tau}^{1:M}; \omega_o)$, where all local latent parameters are assumed to be conditionally independent.
For each local latent parameter $\mathbf{w}^m_{\tau, o}$, the transformer-structured inference module $\phi$ takes $[\bar{x}^{1:M}_{\tau, o}, \omega_{o}]$ as input and outputs the corresponding prior distribution, where $\bar{x}^{1:M}_{\tau, o}$ are deep features from the same class $o$ in the episode and $\omega_{o}$ is the corresponding learnable token. Here the inference steps for $\mathbf{w}^m_{\tau, o}$ are as follows:
\begin{equation}
    [\widetilde{x}^{1:M}_{\tau, o}; \widetilde{\omega}_o] = \texttt{MSA}( \texttt{LN} ([\bar{x}^{1:M}_{\tau, o}; \omega_o])) + [\bar{x}^{1:M}_{\tau, o}; \omega_o],
\label{eq: w_1}
\end{equation}
\begin{equation}
    [\widehat{x}^{1:M}_{\tau, o}; \widehat{\omega}_o] = \texttt{MLP}( \texttt{LN} ([\widetilde{x}^{1:M}_{\tau, o}; \widetilde{\omega}_o])) + [\widetilde{x}^{1:M}_{\tau, o}; \widetilde{\omega}_o],
\label{eq: w_2}
\end{equation}
\begin{equation}
\begin{aligned}
     p_{\phi}(\mathbf{w}^m_{\tau, o}  | \mathbf{z}_{\tau}^m, \mathcal{C}_{\tau}^{1:M}; \omega_o)  =  \mathcal{N}(\mathbf{w}^m_{\tau, o}; 
    \mu_{\mathbf{w}^m_{\tau, o}},\sigma_{\mathbf{w}^m_{\tau, o}}),
\end{aligned}
\label{eq: w_3}
\end{equation}where
$\mu_{\mathbf{w}^m_{\tau, o}}=\texttt{MLP}(\widehat{\omega}_o, {\mathbf{z}^{m}_{\tau}}^{(i)}), 
\sigma_{\mathbf{w}^m_{\tau, o}}=
\texttt{Softplus}(\texttt{MLP}(\widehat{\omega}_o, {\mathbf{z}^{m}_{\tau}}^{(i)}))$. ${\mathbf{z}^{m}_{\tau}}^{(i)}$ is a Monte Carlo sample from the variational posterior of the corresponding global latent representation during meta-training. 

Both transformer-structured inference modules use the refined tokens $\widehat{\nu}^m$ and $\widehat{\omega}_o$ to obtain a global latent representation and a local latent parameter, respectively. The introduced tokens preserve the specific meta-knowledge for each latent variable during inference. Compared with the $\theta$-parameterised inference module exploring the intra-task relationships, the $\phi$-parameterised inference module enables the exploitation of the inter-task relationships to reason over each local latent parameter. Thus, the introduced tokens can be refined with relevant information from the heterogeneous context sets. By integrating meta-knowledge and heterogeneous context sets, HNPs can reduce the negative transfer of task-specific knowledge among heterogeneous tasks in each episode. Please refer to Appendix E for algorithms.

%% file: 2_related.tex
\section{Related Work}
\label{sec: related work}

\paragraph{Multi-Task Learning.}
Multi-task learning can operate in various settings~\cite{zhang2021survey}. Here we roughly separate the settings of MTL into two branches: 
(1) Single-input multi-output (SIMO)~\cite{kendall2018multi, liu2019end, misra2016cross, liu2020towards, sener2018multi, zamir2018taskonomy, bhattacharjee2022mult}, where tasks are defined by different supervision information for the same input. 
(2) Multi-input multi-output (MIMO)~\cite{shen2021variational, long2017learning, zhang2021multi, shen2022association, liang2022scheduled, zhang2022learning}, where heterogeneous tasks follow different data distributions. 
This work considers the MIMO setup of multi-task learning with episodic training.

In terms of modeling methods, from a processing perspective, existing MTL methods can be roughly categorized into two groups: (1)~Probabilistic MTL methods~\cite{shen2021variational, kim2021multi, bakker2003task, yu2005learning, titsias2011spike, lawrence2004learning, yousefi2019multi, oyen2012leveraging, qian2020multi}, which employ the Bayes framework to characterize probabilistic dependencies among tasks. 
(2)~Deep MTL models ~\cite{long2017learning, shen2022association, kendall2018multi, liu2019end, misra2016cross, strezoski2019learning, sun2019adashare, strezoski2019many, yu2020gradient, liu2021conflict, fifty2021efficiently, guo2020learning}, which directly utilize deep neural networks to discover information-sharing mechanisms across tasks. 
However, deep MTL models rely on large amounts of training data and tend to overfit when encountering the data-insufficiency problem. 
Meanwhile, previous probabilistic MTL methods consider a small number of tasks that occur at the same time, limiting their applicability in real-world systems. 

\paragraph{Meta-Learning.}
Meta-learning aims to find strategies to quickly adapt to unseen tasks with a few examples ~\cite{thrun1998learning, vinyals2016matching, finn2017model, hospedales2021meta}. 
There exist a couple of branches in meta-learning methods, such as metrics-based 
methods~\cite{snell2017prototypical, allen2019infinite, oreshkin2018tadam, yoon2019tapnet, garcia2018few,cao2019theoretical,triantafillou2019meta, sung2018learning} and optimization-based methods~\cite{finn2017model, nichol2018first, rusu2018meta, sung2017learning, li2017meta, gordon2018meta, edwards2016towards,finn2018probabilistic, saemundsson2018meta, baik2021learning, choi2021test, raghu2019rapid}.
Our paper focuses on a probabilistic meta-learning method, namely neural processes, that can quantify predictive uncertainty.
Models in this family~\cite{requeima2019fast, garnelo2018neural, garnelo2018conditional,   kim2019attentive, wang2022learning, wang2020doubly, wei2021meta, markou2022practical, ye2022contrastive, kim2022neural, wang2023bridge} can approximate stochastic processes in neural networks. 
Vanilla NPs~\cite{garnelo2018neural} usually encounter the expressiveness bottleneck because their functional priors are not rich enough to generate complicated functions~\cite{kim2019attentive, wang2022learning}.
~\cite{requeima2019fast} introduces deterministic variables to model predictive distributions for meta-learning scenarios directly.
Most NP-based methods only focus on a single task during inference~\cite{requeima2019fast, garnelo2018neural, kim2019attentive, wang2022learning, bruinsma2023autoregressive}, which leaves task-relatedness between heterogeneous tasks in a single episode an open problem. 

This paper combines multi-task learning and meta-learning paradigms to tackle the data-insufficiency problem.
Our work shares the high-level goal of exploiting task-relatedness in an episode with~\cite{kim2021multi, song2022efficient, upadhyay2023multi}. 
Concerning the multi-task scenarios, the main differences are: ~\cite {kim2021multi, song2022efficient, upadhyay2023multi} handles multiple attributes and multi-sensor data under the SIMO setting, while our work performs for the MIMO setting where tasks are heterogeneous and distribution shifts exist. 
Moreover, \cite{wang2021bridging} theoretically addresses the conclusion that MTL methods are powerful and efficient alternatives to gradient-based meta-learning algorithms. However, our method inherits the advantages of multi-task learning and meta-learning: simultaneously generalizing meta-knowledge from past to new episodes and exploiting task-relatedness across heterogeneous tasks in every single episode. Thus, our method is more suitable for solving the data-insufficiency problem. Intuitive comparisons with related paradigms such as~\textit{cross-domain few-shot learning}~\cite{tseng2020cross, chen2019closer, guo2020broader, du2020metanorm, das2022confess, li2022cross}, \textit{multimodal meta-learning}~\cite{ vuorio2019multimodal, vuorio2018toward, yao2019hierarchically, abdollahzadeh2021revisit, chen2021hetmaml, triantafillou2019meta} and~\textit{cross-modality few-shot learning}~\cite{xing2019adaptive, pahde2018cross, pahde2021multimodal} are provided in Appendix A.

%% file: 4_experiments.tex
\section{Experiments}
\label{sec: experiment}

We evaluate the proposed HNPs and baselines on three benchmark datasets under the episodic multi-task setup. 
Sec.~\ref{experiment: MMTL_regression} and Sec.~\ref{experiment: MMTL_classification} provide experimental results for regression and classification, respectively. Ablation studies are in Sec.~\ref{experiment: MMTL_ablation}. 
More comparisons with recent works on extra datasets are provided in Appendix F. Additional results under the convectional MIMO setup without episodic training can be found in Appendix G \& H.

\subsection{Episodic Multi-Task Regression}
\label{experiment: MMTL_regression}

\paragraph{Dataset and Settings.} 
To evaluate the benefit of HNPs over typical NP baselines in uncertainty quantification, we conduct experiments in several 1D regression tasks.
The baselines include conditional neural processes (CNPs~\cite{garnelo2018conditional}), vanilla neural processes (NPs~\cite{garnelo2018neural}), and attentive neural processes (ANPs~\cite{kim2019attentive}). 
As a toy example, we construct multiple tasks with different task distributions: each task's input set is defined on separate intervals without overlap. 

\begin{wrapfigure}{r}{0.50\textwidth}
\centering
\vspace{-4mm}
\includegraphics[width=1.0\linewidth]{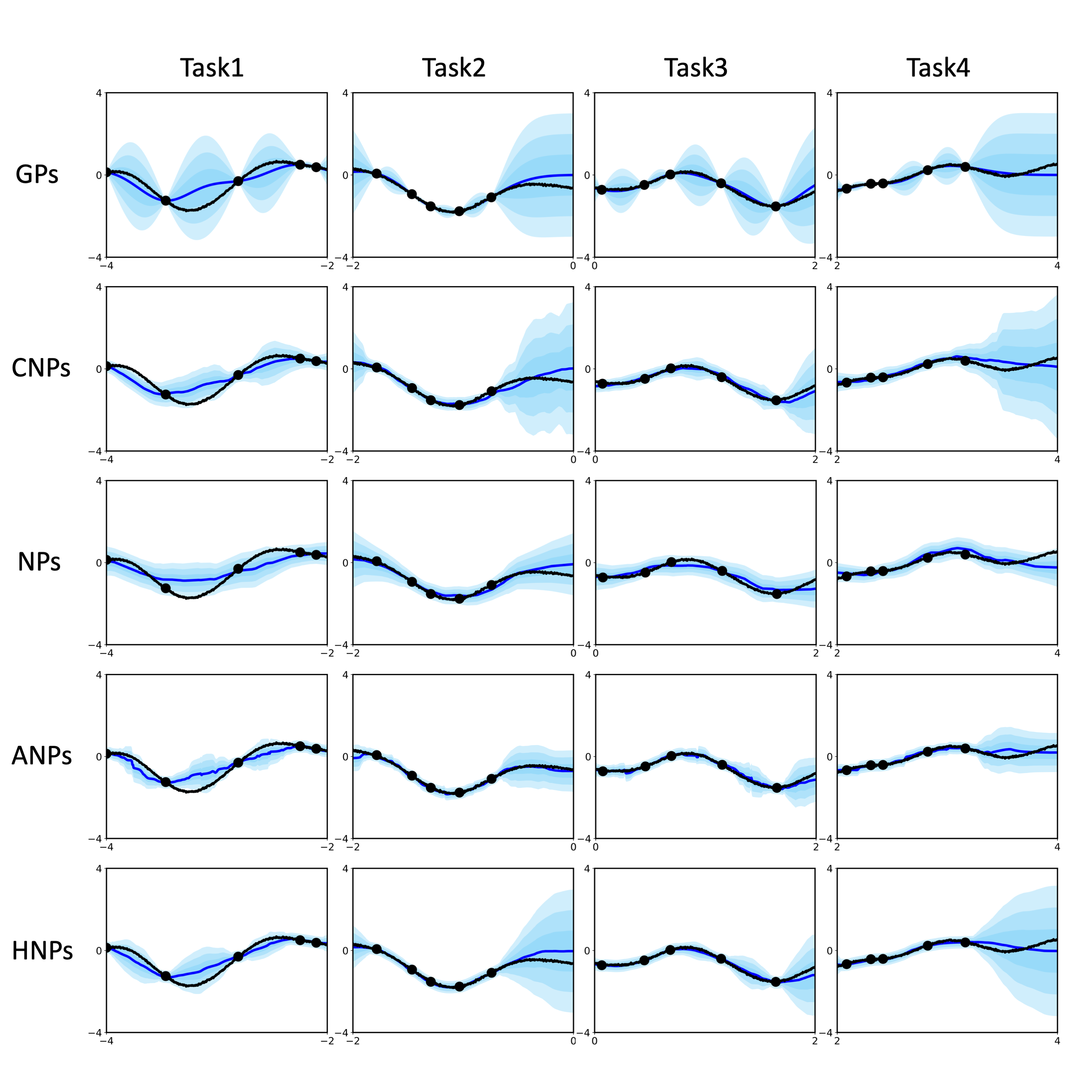}
\vspace{-6mm}
\caption{\textbf{Performance comparisons on the episodic multi-task $1$-D function regression using $5$ context points (black dots) for each task.} Black curves are ground truth, and blue ones are predicted results. The shadow regions are $\pm3$ standard derivations from the mean~\cite{wang2020doubly}.
}
\label{fig: regression}
\vspace{-4mm}
\end{wrapfigure}

Given four different tasks in an episode, their input sets are $\mathbf{x}^{1:4}_{\tau}$. 
Each input set contains a few instances, drawn uniformly at random from separate intervals, such as $\mathbf{x}^1_{\tau} \in [-4, -2)$, $\mathbf{x}^2_{\tau} \in [-2, 0)$, $\mathbf{x}^3_{\tau} \in [0, 2)$, and $\mathbf{x}^4_{\tau} \in [2, 4)$. 
All tasks in an episode are related by sharing the same ground truth function. 
Following~\cite{garnelo2018neural, wang2020doubly}, function-fitting tasks are generated with Gaussian processes~(GPs). 
Here a zero mean Gaussian process $y^{(0)} \sim \mathcal{GP}(0, k(\cdot, \cdot))$ is used to produce $\mathbf{y}^{1:4}_{\tau}$ for the inputs from all tasks $\mathbf{x}^{1:4}_{\tau}$. A radial basis kernel $k(x, x{'}) = \sigma^2\exp(-(x-x{'})^2)/2l^2)$, with $l=0.4$ and $\sigma = 1.0$ is used. 

\paragraph{Results and Discussions.}
As shown in Figure~\ref{fig: regression}, CNPs, ANPs, and our HNPs exhibit more reasonable uncertainty than NPs in Figure~\ref{fig: regression}: lower variances are predicted around observed (context) points with higher variances around unobserved points.
Furthermore, NPs and ANPs detrimentally impact the smoothness of the predicted curves, whereas HNPs yield smoother predictive curves with reliable uncertainty estimation. These observations suggest that integrating correlation information across related tasks and meta-knowledge in HNPs can improve uncertainty quantification in multi-task regression.

\begin{wraptable}{r}{0.50\textwidth}
\vspace{-6mm}
\caption{\textbf{Average negative log-likelihoods over target points from all tasks.}}
\centering
\vspace{+2mm}
\begin{adjustbox}{max width =1.0\linewidth}
\begin{tabular}{lcccc}
\toprule
Methods & CNPs & NPs & ANPs & HNPs \\
\midrule
Avg. NLL & 0.0935 & 0.8649 & -0.1165 & -0.5207 \\
\bottomrule
\end{tabular}
\end{adjustbox}
\vspace{-4mm}
\label{tab: rebuttal8}
\end{wraptable} 
To quantify uncertainty we use the average negative log-likelihood (the lower, the better). As shown in Table~\ref{tab: rebuttal8}, our HNPs achieve a lower average negative log-likelihood than baselines, demonstrating our method's effectiveness in uncertainty estimation.

\subsection{Episodic Multi-task Classification} 
\label{experiment: MMTL_classification}

\paragraph{Datasets and Settings.}

We use \texttt{Office-Home}~\cite{venkateswara2017Deep} and \texttt{DomainNet}~\cite{peng2019moment} as episodic multi-task classification datasets. \texttt{Office-Home} contains images from four domains: Artistic~(A), Clipart~(C), Product~(P) and Real-world~(R). Each domain contains images from  $65$ categories collected from office and home environments. Note that all domains share the whole target space. The numbers of meta-training classes and meta-test classes are $40$ and $25$, respectively.
There are about $15,500$ images in total. 
\texttt{DomainNet} has six distinct domains: Clipart, Infograph, Painting, Quickdraw, Real and Sketch. It includes approximately 0.6 million images distributed over $345$ categories. The numbers of meta-training classes and meta-test classes are $276$ and $69$, respectively.
Here one domain corresponds to a specific task distribution in the episodic multi-task setting. 

When it comes to the episodic multi-task classification, we compare HNPs with the following three branches: 
(1)~\textit{Multi-task learning methods}: ERM~\cite{gulrajani2020search} directly expands the training set of the current task with samples of related tasks. 
VMTL~\cite{shen2021variational} is one of the state-of-the-art under the MIMO setting of multi-task learning.  
(2)~\textit{Meta-learning methods}: MAML~\cite{finn2017model}, Proto.Net~\cite{snell2017prototypical} and DGPs~\cite{wang2021learning} address each task separately with no mechanism to leverage task-relatedness in a single episode.
(3)~\textit{Methods from the NP family}: CNPs~\cite{garnelo2018conditional} and NPs~\cite{garnelo2018neural} are established methods in the NP family. TNP-D~\cite{nguyen2022transformer} is recent NP work in sequential decision-making for a single task in each episode. 

\paragraph{Results and Discussions.}
The experimental results for episodic multi-task classification on \texttt{Office-Home} and \texttt{DomainNet} are reported in Table~\ref{tab:meta_office-home}.  
We use the average accuracy across all task distributions as the evaluation metric. 
It can be seen that HNPs consistently outperform all baseline methods, demonstrating the effectiveness of HNPs in handling each task with limited data under the episodic multi-task classification setup.

NPs and CNPs do not work well under all episodic multi-task classification cases. This can be attributed to their limited expressiveness of the global representation and the weak capability to extract discriminative information from multiple contexts. 
In contrast, HNPs explicitly abstract discriminative information for each task in the episode with the help of local latent parameters, enhancing the expressiveness of the functional prior.

\begin{table}[t]
\caption{\textbf{Comparative results (95\% confidence interval) for episodic multi-task classification on  \texttt{Office-Home} and  \texttt{DomainNet}.} Best results are indicated in bold. }
\centering
\begin{adjustbox}{max width =1.0\linewidth}
\begin{tabular}{l|cccc|cccc}
\toprule
&\multicolumn{4}{c|}{\texttt{Office-Home}} & \multicolumn{4}{c}{\texttt{DomainNet}}  \\
\cline{2-9}
&\multicolumn{2}{c}{\texttt{4-task 5-way}} & \multicolumn{2}{c|}{\texttt{4-task 20-way}}  &\multicolumn{2}{c}{\texttt{6-task 5-way}} & \multicolumn{2}{c}{\texttt{6-task 20-way}}  \\
Method & \texttt{1-shot} & \texttt{5-shot} & \texttt{1-shot} & \texttt{5-shot} & \texttt{1-shot} & \texttt{5-shot} & \texttt{1-shot} & \texttt{5-shot} \\
\midrule
ERM
\cite{gulrajani2020search}
&66.04 \scriptsize{$\pm$0.61}
&73.62 \scriptsize{$\pm$0.55}	
&39.25 \scriptsize{$\pm$0.24}	
&47.14 \scriptsize{$\pm$0.18}
&59.95 \scriptsize{$\pm$0.52}
&68.52 \scriptsize{$\pm$0.44}
&38.62 \scriptsize{$\pm$0.22}	
&47.85 \scriptsize{$\pm$0.20}
\\
VMTL
\cite{shen2021variational}
&49.71 \scriptsize{$\pm$0.48}
&65.75 \scriptsize{$\pm$0.47}
&27.50 \scriptsize{$\pm$0.14}
&42.82 \scriptsize{$\pm$0.13}
&42.24 \scriptsize{$\pm$0.39}
&57.37 \scriptsize{$\pm$0.43}
&18.05 \scriptsize{$\pm$0.11}
&31.38 \scriptsize{$\pm$0.15}
\\
\midrule
MAML 
\cite{finn2017model}   
&60.58 \scriptsize{$\pm$0.60} 
&75.29 \scriptsize{$\pm$0.53} 
&34.29 \scriptsize{$\pm$0.19} 
&48.39 \scriptsize{$\pm$0.20} 
&53.21	\scriptsize{$\pm$0.46}	
&65.24	\scriptsize{$\pm$0.47}	
&17.10	\scriptsize{$\pm$0.12}	
&20.35	\scriptsize{$\pm$0.14} 
\\
Proto. Net.
\cite{snell2017prototypical}
&57.19 \scriptsize{$\pm$0.53}	
&74.97 \scriptsize{$\pm$0.46}
&32.72 \scriptsize{$\pm$0.18}
&49.75 \scriptsize{$\pm$0.16}
&53.71 \scriptsize{$\pm$0.48}	
&68.80 \scriptsize{$\pm$0.42}
&31.90 \scriptsize{$\pm$0.19}
&47.59 \scriptsize{$\pm$0.18}
\\
DGPs 
\cite{wang2021learning}
&65.89 \scriptsize{$\pm$0.53} 
&79.96 \scriptsize{$\pm$0.38} 
&31.48 \scriptsize{$\pm$0.18}
&49.46 \scriptsize{$\pm$0.18}  
&50.93 \scriptsize{$\pm$0.42} 
&63.32 \scriptsize{$\pm$0.38} 
&25.46 \scriptsize{$\pm$0.15} 
&38.63 \scriptsize{$\pm$0.17} 
\\
\midrule
CNPs 
\cite{garnelo2018conditional} 
&43.33	\scriptsize{$\pm$0.56}
&55.07	\scriptsize{$\pm$0.63}
&10.57	\scriptsize{$\pm$0.10}
&12.02	\scriptsize{$\pm$0.11}
&37.90	\scriptsize{$\pm$0.45}
&40.53	\scriptsize{$\pm$0.44}
&5.12	\scriptsize{$\pm$0.10}
&5.14	\scriptsize{$\pm$0.10}
\\
NPs
\cite{garnelo2018neural}
&33.66	\scriptsize{$\pm$0.48}
&53.99	\scriptsize{$\pm$0.60}
&5.25	\scriptsize{$\pm$0.16}
&11.40	\scriptsize{$\pm$0.11}
&20.58	\scriptsize{$\pm$0.51}
&20.53	\scriptsize{$\pm$0.53}
&5.12	\scriptsize{$\pm$0.09}
&5.11	\scriptsize{$\pm$0.09}
\\ 
TNP-D
\cite{nguyen2022transformer}
&65.49	\scriptsize{$\pm$0.53}
&78.94  \scriptsize{$\pm$0.43}
&41.61	\scriptsize{$\pm$0.22}
&59.19  \scriptsize{$\pm$0.21}
&49.10 \scriptsize{$\pm$0.42}
&67.39 \scriptsize{$\pm$0.40}
&28.83 \scriptsize{$\pm$0.17}
&47.69 \scriptsize{$\pm$0.18}
\\
\rowcolor{LightGray}
HNPs
&\textbf{76.29} {\scriptsize{$\pm$0.51}}
&\textbf{80.80} {\scriptsize{$\pm$0.42}}
&\textbf{51.82} {\scriptsize{$\pm$0.23}}	
&\textbf{59.97} {\scriptsize{$\pm$0.18}} 
&\textbf{62.36}	\scriptsize{$\pm$0.53}	
&\textbf{69.38}	\scriptsize{$\pm$0.42}	
&\textbf{39.32}	\scriptsize{$\pm$0.23}	
&\textbf{48.56}	\scriptsize{$\pm$0.19} 
\\
\bottomrule
\end{tabular}
\end{adjustbox}
\label{tab:meta_office-home}
\vspace{-6mm}
\end{table}

We also find that HNPs significantly surpass other baselines on \texttt{1-shot} \texttt{Office-Home} and \texttt{DomainNet}, both under the \texttt{4/6-task 5-way} and \texttt{4/6-task 20-way} settings. 
This further implies that HNPs can circumvent the effect of the problem of data-insufficiency by simultaneously exploiting task-relatedness across heterogeneous tasks and meta-knowledge among episodes.

\subsection{Ablation Studies}
\label{experiment: MMTL_ablation}

\paragraph{Influence of Hierarchical Latent Variables.} 
We first investigate the roles of the global latent representation $\mathbf{z}^m_{\tau}$ and the local latent parameters $\mathbf{w}^m_{\tau, 1:O}$ by leaving out individual inference modules. 
These experiments are performed on $\texttt{Office-home}$ under the \texttt{4-task 5-way 1-shot} setting. 
We report the detailed performance for tasks sampled from a single task distribution (A/C/P/R) and the average accuracy across all task distributions (Avg.) in Table~\ref{tab:ablation_both}. 
The variants without specific latent variables are included in the comparison by removing the corresponding inference modules. 

\begin{wraptable}{r}{0.60\textwidth}
\vspace{-6mm}
\caption{\textbf{Effectiveness of global latent representations $\mathbf{z}_{\tau}^{m}$ and local latent parameters $\mathbf{w}_{\tau, 1:O}^{m}$ in the model.} \ymark~and~\xmark~denote whether the variants of HNPs have the corresponding latent variable or not.}
\centering
\vspace{+2mm}
\begin{adjustbox}{max width =1.0\linewidth}
\begin{tabular}{ccccccc}
\toprule
$\mathbf{z}_{\tau}^{m}$ &  $\mathbf{w}_{\tau, 1:O}^{m}$ & \textbf{A} & \textbf{C} & \textbf{P} & \textbf{R} & \textbf{Avg.} \\ 
\midrule
\xmark & \xmark  
& 62.64~{\scriptsize$\pm$0.72} & 56.87~{\scriptsize$\pm$0.71} & 75.18~{\scriptsize$\pm$0.79} & 73.68~{\scriptsize$\pm$0.77} & 67.09~{\scriptsize$\pm$0.63} \\ 
\xmark & \ymark 
& 69.39~{\scriptsize$\pm$0.60} & 63.10~{\scriptsize$\pm$0.61} & 80.66~{\scriptsize$\pm$0.67} & 79.99~{\scriptsize$\pm$0.62} & 73.29~{\scriptsize$\pm$0.51} \\ 
\ymark & \xmark 
& 67.02~{\scriptsize$\pm$0.67} & 60.70~{\scriptsize$\pm$0.69} & 78.26~{\scriptsize$\pm$0.76} & 78.47~{\scriptsize$\pm$0.72} & 71.11~{\scriptsize$\pm$0.59} \\ 
\rowcolor{LightGray}
\ymark & \ymark 
& \textbf{73.31~{\scriptsize$\pm$0.63}} & \textbf{64.92~{\scriptsize$\pm$0.68}} & \textbf{83.38~{\scriptsize$\pm$0.66}} & \textbf{83.54~{\scriptsize$\pm$0.64}} & \textbf{76.29~{\scriptsize$\pm$0.51}} \\
\bottomrule
\end{tabular}
\end{adjustbox}
\label{tab:ablation_both}
\vspace{-2mm}
\end{wraptable}

As shown in Table~\ref{tab:ablation_both}, both $\mathbf{z}^m_{\tau}$ and $\mathbf{w}^m_{\tau, 1:O}$ benefit overall performance. 
Our method with hierarchical latent variables performs $9.20\%$ better than the variant without both latent variables,  $3.00\%$ better than the variant without $\mathbf{z}^m_{\tau}$, and $5.18\%$ better than the variant without $\mathbf{w}^m_{\tau, 1:O}$. 
This indicates that latent variables of HNPs complement each other in representing context sets from multiple tasks and meta-knowledge. 
The variant without $\mathbf{w}^m_{\tau, 1:O}$ underperforms the variant without $\mathbf{z}^m_{\tau}$ by $2.18\%$, in terms of the average accuracy. This demonstrates that $\mathbf{z}^m_{\tau}$ suffers more from the expressiveness bottleneck than $\mathbf{w}^m_{\tau, 1:O}$, weakening the models' discriminative ability. For classification, local latent parameters are more crucial than a global latent representation in revealing the discriminating knowledge from multiple heterogeneous context sets.

\paragraph{Influence of Transformer-Structured Inference Modules.} 

To further understand our transformer-structured inference modules (Trans. w learnable tokens), we examine the performance against two other options: inference modules that solely utilize a multi-layer perceptron (MLP) and the variants that do not incorporate any learnable tokens (Trans. w/o learnable tokens). We also compare the probabilistic and deterministic versions of such inference modules. The deterministic variants consider the deterministic embedding for the hierarchical latent variables. 

\begin{wraptable}{r}{0.60\textwidth}
\vspace{-6mm}
\caption{\textbf{Performance comparisons between our transformer inference modules (Trans. w learnable tokens) and other alternatives.}}
\centering
\vspace{+2mm}
\begin{adjustbox}{max width =1.0\linewidth}
\begin{tabular}{cc|cc}
\toprule
\multicolumn{2}{c|}{Inference networks}  & 1-shot & 5-shot \\ 
\midrule
\multirow{3}{*}{Deterministic} 
&MLP
& 64.93~{\scriptsize$\pm$0.66} & 72.39~{\scriptsize$\pm$0.56} \\
&Trans. w/o learnable tokens 
& 70.22~{\scriptsize$\pm$0.62} &  76.15~{\scriptsize$\pm$0.54}\\ 
&Trans. w learnable tokens
& 70.61~{\scriptsize$\pm$0.56} &  76.70~{\scriptsize$\pm$0.50}\\
\midrule
\multirow{3}{*}{Probabilistic}
&MLP
& 73.30~{\scriptsize$\pm$0.59} & 77.94~{\scriptsize$\pm$0.48} \\
&Trans. w/o learnable tokens  
& 75.25~{\scriptsize$\pm$0.55} & 80.42~{\scriptsize$\pm$0.47} \\ 
&Trans. w learnable tokens
& \textbf{76.29~{\scriptsize$\pm$0.51}} &\textbf{80.80~{\scriptsize$\pm$0.42}}  \\
\bottomrule
\end{tabular}
\end{adjustbox}
\vspace{-4mm}
\label{tab:ablation_inference}
\end{wraptable}
As shown in Table~\ref{tab:ablation_inference}, our inference modules consistently outperform the variants, regardless of whether the inference network is probabilistic or deterministic. When using the probabilistic one, our inference modules respectively achieve $1.04\%$ and $2.99\%$ performance gains over Trans. w/o learnable tokens and MLP under the \texttt{4-task 5-way 1-shot} setting. 
This implies the importance of learnable tokens and task-relatedness in formulating transformer-structured inference modules, which reduces negative transfer among heterogeneous tasks in each meta-test episode. Moreover, the variants with probabilistic inference modules consistently beat deterministic ones in performance, demonstrating the advantages of considering uncertainty during modeling and inference. 

\begin{wraptable}{r}{0.60\textwidth}
\vspace{-6mm}
\caption{\textbf{Performance comparisons of different implementations of generating each local latent parameter $\mathbf{w}^m_{\tau, o}$ from the condition $\mathbf{z}^m_{\tau}$ and $\mathcal{C}_{\tau}^{1:M}$.}}
\centering
\vspace{+2mm}
\begin{adjustbox}{max width =1.0\linewidth}
\begin{tabular}{cccccc}
\hline
\toprule
Methods & \textbf{A} & \textbf{C} & \textbf{P} & \textbf{R} & \textbf{Avg.} \\ 
\midrule
\texttt{Concat} 
& 65.69~{\scriptsize$\pm$0.59} & 58.64~{\scriptsize$\pm$0.61} & 77.54~{\scriptsize$\pm$0.68} & 77.10~{\scriptsize$\pm$0.64} & 69.74~{\scriptsize$\pm$0.51} \\ 
\texttt{Add} 
& 69.92~{\scriptsize$\pm$0.69} & 63.73~{\scriptsize$\pm$0.71} & 78.81~{\scriptsize$\pm$0.77} & 79.03~{\scriptsize$\pm$0.78} & 72.87~{\scriptsize$\pm$0.61} \\ 
\rowcolor{LightGray}
\texttt{Ours} 
& \textbf{73.31~{\scriptsize$\pm$0.63}} & \textbf{64.92~{\scriptsize$\pm$0.68}} & \textbf{83.38~{\scriptsize$\pm$0.66}} & \textbf{83.54~{\scriptsize$\pm$0.64}} & \textbf{76.29~{\scriptsize$\pm$0.51}} \\
\bottomrule
\end{tabular}
\end{adjustbox}
\vspace{-4mm}
\label{tab:ablation_hierarchical}
\end{wraptable}

\paragraph{Effects of Different Ways to Generate Local Latent Parameters.} 
We investigate the effects of different ways to generate each $\mathbf{w}^m_{\tau, o}$ from the shared condition $\mathbf{z}^m_{\tau}$ and $\mathcal{C}_{\tau}^{1:M}$. Given a Monte Carlo sample of global latent variables as ${\mathbf{z}^m_{\tau}}^{(i)}$, in Table~\ref{tab:ablation_hierarchical}, we compare with two alternatives:
$\texttt{1)~Concat}$ directly concatenates each context feature and ${\mathbf{z}^m_{\tau}}^{(i)}$, and takes the concatenation as inputs of the transformer-structured inference network $\phi$.
$\texttt{2)~Add}$ sums up each context feature and  ${\mathbf{z}^m_{\tau}}^{(i)}$ and takes the result as the input. 
$\texttt{3)~Ours}$ incorporates ${\mathbf{z}^m_{\tau}}^{(i)}$ into the transformer-structured inference module by merging it with the refined learnable tokens in Eq. (\ref{eq: w_3}). As shown in Table ~\ref{tab:ablation_hierarchical}, $\texttt{Ours}$ consistently performs the best. This implies that incorporating the conditional variables into the inference module is more effective than the direct combinations of ${\mathbf{z}^m_{\tau}}^{(i)}$ and instance features.

\paragraph{Effects of More "Shots" or "Classes".}

To investigate the effects of more "shots" or "classes" in the episodic multi-task classification setup, we conduct experiments by increasing $K$ or $O$ in the defined \texttt{$M$-task $O$-way $K$-shot} setup. 

\begin{wraptable}{r}{0.60\textwidth}
\vspace{-7mm}
\caption{\textbf{Performance comparisons on \texttt{Office-Home} under the \texttt{4-task 5-way K-shot} setup.}}
\centering
\vspace{+1mm}
\begin{adjustbox}{max width =1.0\linewidth}
\begin{tabular}{ccccc}
\toprule
Methods & \texttt{1-shot} & \texttt{5-shot} & \texttt{10-shot} & \texttt{20-shot} \\
\midrule
TNP-D & 65.49~{\scriptsize$\pm$0.53} & 78.94~{\scriptsize$\pm$0.43} & 80.81~{\scriptsize$\pm$0.32} & 81.12~{\scriptsize$\pm$0.68} \\
\rowcolor{LightGray}
HNPs & 76.29~{\scriptsize$\pm$0.51} & 80.80~{\scriptsize$\pm$0.42} & 81.28~{\scriptsize$\pm$0.38} & 81.56~{\scriptsize$\pm$0.36}\\
\bottomrule
\end{tabular}
\end{adjustbox}
\vspace{-4mm}
\label{tab: rebuttal1}
\end{wraptable} 

As shown in Table~\ref{tab: rebuttal1}, the proposed HNPs have more advantages over the baseline method with the context data points below ten shots. With shots larger than ten, both methods will reach a performance bottleneck.

Moreover, Table~\ref{tab: rebuttal2} shows that our method consistently outperforms the baseline method as the number of classes increases from 20 to 40 in step 5. However, the performance gap between them narrows slightly with more classes. The main reason could be that the setting with more classes suffers from less data insufficiency.
\begin{table}[ht]
\vspace{-2mm}
\caption{\textbf{Performance comparisons on \texttt{DomainNet} under the \texttt{6-task O-way 1-shot} setup.}}
\centering
\vspace{-2mm}
\begin{adjustbox}{max width =1.0\linewidth}
\begin{tabular}{lcccccc}
\toprule
Methods & \texttt{5-way} & \texttt{20-way} & \texttt{25-way} & \texttt{30-way} & \texttt{35-way} & \texttt{40-way} \\
\midrule
TNP-D & 49.10~{\scriptsize$\pm$0.42} & 28.83~{\scriptsize$\pm$0.17} & 25.93~{\scriptsize$\pm$0.14} & 24.08~{\scriptsize$\pm$0.12} & 22.62~{\scriptsize$\pm$0.11} & 21.64~{\scriptsize$\pm$0.53} \\
\rowcolor{LightGray}
HNPs & 62.36~{\scriptsize$\pm$0.53} & 39.32~{\scriptsize$\pm$0.23} & 35.72~{\scriptsize$\pm$0.19} & 32.27~{\scriptsize$\pm$0.17} & 31.27~{\scriptsize$\pm$0.14} & 29.31~{\scriptsize$\pm$0.13}\\
\bottomrule
\end{tabular}
\end{adjustbox}
\vspace{-4mm}
\label{tab: rebuttal2}
\end{table}

\begin{wrapfigure}{r}{0.6\textwidth}
\centering
\vspace{-4mm}
\includegraphics[width=1.0\linewidth]{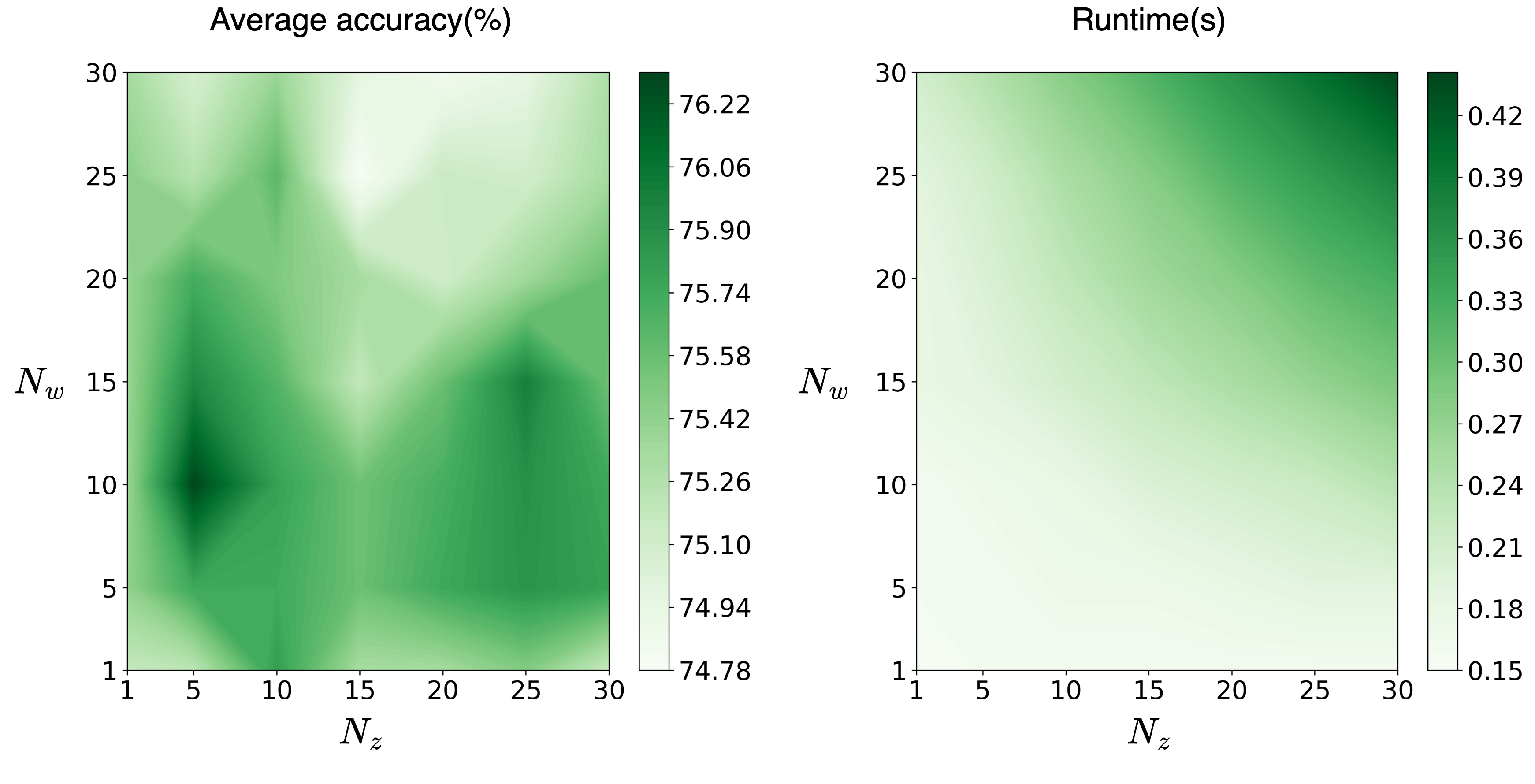}
\vspace{-4mm}
\caption{\textbf{Average accuracy and runtime of HNPs with different numbers of Monte Carlo samples.} $N_z$ and $N_w$ are sampling numbers of $\mathbf{z}^m_{\tau}$ and $\mathbf{w}^m_{\tau, 1:O}$, respectively.
}
\vspace{-4mm}
\label{fig:ablation_sampling}
\end{wrapfigure}
\paragraph{Sensitivity to the Number of Monte Carlo Samples.} For the hierarchical latent variables in the HNPs, we investigate the model's sensitivity to the number of Monte Carlo samples. Specifically, the sampling number of the global latent representation $\mathbf{z}^{m}_{\tau}$ and local latent parameters $\mathbf{w}^{m}_{\tau, 1:O}$ varies from $1$ to $30$. 
We examine on \texttt{Office-Home} under the \texttt{4-task 5-way 1-shot} setting.
In Figure~\ref{fig:ablation_sampling}, the runtime per iteration grows rapidly as the number of samples increases. However, there is no clear correlation between the performance and the number of Monte Carlo samples.
There are two sweet spots in terms of average accuracy, one of which has favorable computation time. Hence, we set $N_z$ and $N_w$ to $5$ and $10$, respectively.

\begin{wraptable}{r}{0.60\textwidth}
\vspace{-4mm}
\caption{\textbf{Inference time of different NP-based methods.}}
\centering
\vspace{+2mm}
\begin{adjustbox}{max width =1.0\linewidth}
\begin{tabular}{lcccc}
\toprule
Methods & CNPs & NPs & TNP-D & HNPs \\
\midrule
Inference time(s) & 0.04 & 0.05 & 0.08 & 0.15 \\
\bottomrule
\end{tabular}
\end{adjustbox}
\vspace{-4mm}
\label{tab: rebuttal6}
\end{wraptable} 

We also investigate the inference time of NP-based models per iteration on \texttt{ Office-Home} under the 4task5way1shot setup. As shown in Table~\ref{tab: rebuttal6}, our model needs more inference time than other NP-based methods for performance gains. The cost mainly comes from inferring the designed hierarchical latent variables; however, we consider this a worthwhile trade-off for the extra performance. 

%% file: 5_conclusion.tex
\section{Conclusion}
\label{sec: conclusion}
\paragraph{Technical Discussion.} 
This work develops heterogeneous neural processes by introducing hierarchical latent variables and transformer-structured inference modules for episodic multi-task learning.
With the help of heterogeneous context information and meta-knowledge, the proposed model can exploit task-relatedness, reason about predictive function distributions, and efficiently distill past knowledge to unseen heterogeneous tasks with limited data.

\paragraph{Limitation \& Extension.}
Although the hierarchical probabilistic framework could mitigate the expressiveness bottleneck, the model needs more inference time than other NP-based methods for performance gains. Besides, the proposed method requires the target space to be the same across all tasks in a single episode. This requirement could limit the method's applicability in realistic scenarios where target spaces may differ across tasks. Our work could be extended to the new case without the shared target spaces, where the model should construct higher-order task-relatedness to improve knowledge sharing among tasks. Our code~{\footnote{https://github.com/autumn9999/HNPs.git}} is provided to facilitate such extensions.



\section*{Acknowledgment}
This work is financially supported by the Inception Institute of Artificial Intelligence, the University of Amsterdam and the allowance Top consortia for Knowledge and Innovation (TKIs) from the Netherlands Ministry of Economic Affairs and Climate Policy.

%% file: 6_appendix.tex
\appendix

\tableofcontents

\section{Frequently Asked Questions}
In this section, we list frequently asked questions from researchers who help proofread this manuscript. These raised questions might also be relevant for others and help in better understanding the paper, so we include more detailed discussions here.

\paragraph{Connections between different settings.}
This work considers the multi-input multi-output setting of multi-task learning under the episodic training mechanism.

As shown in Table~\ref{tab: connection}, we use "Heterogeneous tasks" to distinguish the different branches of multi-task learning: (1) \textit{single-input multi-output}~(SIMO) considers different tasks which have the same input and different supervision information. 
(2) \textit{multi-input multi-output}~(MIMO) considers heterogeneous tasks, which have different inputs and follow different data distributions. All tasks are related since they share the target space. This setting encourages deep models to deal with the insufficient data of each task by aggregating the training data from related tasks in the spirit of data augmentation.

Meanwhile, "Episodic training" is used to describe the data-feeding strategy. 
\textit{Multi-task meta-learning} also benefits from episodic training, but it follows the SIMO setting in every single episode and cannot sufficiently handle heterogeneous tasks. In our work, \textit{episodic multi-task learning} is designed based on the MIMO setting, suffering from distribution shifts between heterogeneous tasks.
In addition, we note that \textit{conventional meta-learning} follows the "Episodic training" mechanism but focuses on single-task learning in each episode. Thus, "Heterogeneous tasks" is not available here (-).
More details are left in Table (\ref{tab: connection}).

\begin{table}[ht]
\caption{\textbf{Connections between different settings.} The symbols \ymark~and~\xmark~ indicate whether or not the specific setting has the corresponding characteristic.
}
\centering
\begin{adjustbox}{max width =1.0\linewidth}
\begin{tabular}{l|c|c|c}
\toprule
Settings & Methods & Episodic training & Heterogeneous tasks\\ 
\midrule
\textit{single-input multi-output}~(SIMO) & ~\cite{kendall2018multi, liu2019end, misra2016cross, liu2020towards, sener2018multi, zamir2018taskonomy, bhattacharjee2022mult} & \xmark & \xmark \\ 
\textit{multi-input multi-output}~(MIMO) &~\cite{shen2021variational, long2017learning, zhang2021multi, shen2022association, liang2022scheduled, zhang2022learning} & \xmark & \ymark \\ 
\textit{conventional meta-learning} &~\cite{vinyals2016matching, finn2017model, snell2017prototypical, requeima2019fast, garnelo2018neural, garnelo2018conditional} & \ymark & - \\ 
\textit{multi-task meta-learning}&~\cite{upadhyay2022multi, zhang2022leaving, cao2021relational,  kim2021multi, zhang2021multi-modal, han2022meta, chen2021mmtl}  & \ymark & \xmark \\
\rowcolor{LightGray}
\textit{episodic multi-task learning} & This paper & \ymark & \ymark \\
\bottomrule
\end{tabular}
\end{adjustbox}
\label{tab: connection}
\vspace{-4mm}
\end{table}

\paragraph{Problem scope.}
In episodic multi-task learning, we restrict the scope of the problem to the case where tasks in the same episode are related and share the same target space.
There are two main reasons: (1) we follow the MIMO setting of multi-task learning in every single episode, where the same target space assures the existence of the knowledge shared among tasks.
(2) As demonstrated in recent works~\cite{achille2019task2vec, nguyen2021probabilistic}, meta-learning tasks generated from the same categories or taxonomic clusters are closer. 
This also implies that tasks with the same target space are related.

\paragraph{Differences from other episodic single-task setups.}

Based on episodic training, there are several approaches related to the setting of episodic multi-task learning: (1)~\textit{cross-domain few-shot learning} addresses few-shot learning under domain shifts~\cite{tseng2020cross}. Several models~\cite{chen2019closer, tseng2020cross, guo2020broader, du2020metanorm, das2022confess, li2022cross} train a model on a single source domain or several source domains and then generalize it to other domains. In contrast, our research emphasizes the domain shifts within individual episodes rather than among them.
(2)~\textit{multimodal meta-learning} extends few-shot learning from a single input-label domain to multiple different input-label domains~\cite{vuorio2019multimodal}. These methods~\cite{vuorio2018toward, vuorio2019multimodal, triantafillou2019meta, yao2019hierarchically, abdollahzadeh2021revisit, chen2021hetmaml} design a meta-learner that can handle tasks from distinct distributions in sequence. Our work centers on simultaneously dealing with several related tasks within a meta-training or meta-test episode.
(3)~\textit{cross-modality few-shot learning}~\cite{xing2019adaptive, pahde2018cross, pahde2021multimodal} leverages semantic information (e.g., word vectors) to augment the performance of visual tasks and not among visual tasks only. 
The aforementioned approaches exclusively address single-task learning per episode, while our work concurrently tackles multiple heterogeneous and related tasks within each meta-training or meta-test episode. Intuitive comparisons with the approaches are shown in Figure~\ref{fig:MMTL_setting_related}.

\begin{figure*}[ht]
\begin{center}
\includegraphics[width=1.0\textwidth]{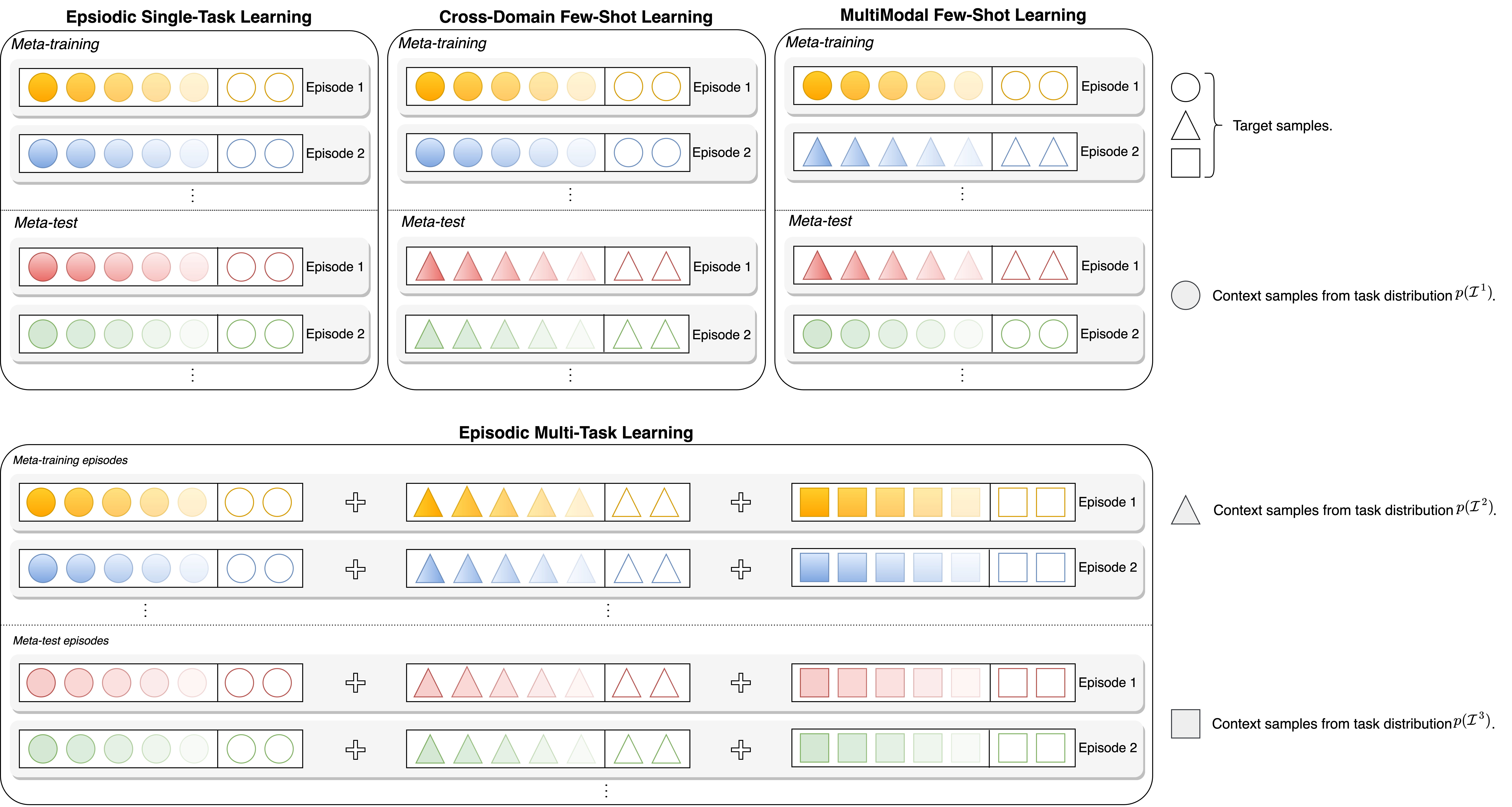}
\end{center}
\caption{\textbf{Differences from other episodic single-task setups}. Each row corresponds to an episode. Different color denotes different categories; the same
color with different shades represents different categories in the same task. Episodic multi-task learning is orthogonal to these setups since it concurrently tackles multiple heterogeneous and related tasks within each episode.}
\label{fig:MMTL_setting_related}
\end{figure*}

\paragraph{Different roles of global and local latent variables.}
In this paper, we introduce global latent representations and local latent parameters within a hierarchical architecture. Each type of them plays a distinct role in the proposed model: (1) Global latent representations provide rich task-specific information during the inference of all local latent parameters. This enables the model to generate task-specific decoders to handle heterogeneous tasks in a single episode effectively. (2) Local latent parameters with prediction-aware information constitute task-specific decoders. Each local latent parameter reveals the knowledge corresponding to a specific prediction across different tasks. This enhances the expressive power of the proposed model. In practice, we observe that global latent representations and local latent parameters complement each other when performing predictions in meta-test episodes. 

\paragraph{Advantages of the proposed hierarchical Bayes framework.}
We summarize the advantages of the proposed framework. (1) A hierarchical Bayesian framework with global and local latent variables yields a richer and more complex latent space to mitigate the expressiveness bottleneck, thus better parameterizing task-specific functions in stochastic processes. (2) Global and local latent variables capture epistemic uncertainty in representation and parameter levels, respectively, and show improved performance in our experiments.

\paragraph{Roles of probabilistic HNPs and KL values in meta training.}
The probabilistic HNPs encode the context as the heterogeneous prior and reveal the uncertainty resulting from data insufficiency and the extent of observations in tasks. Additionally, minimizing KL terms encourages priors inferred from context sets to stay close to posteriors inferred from target sets, guiding more efficient conditional generation.
In training processes, we observed the KL divergence value does not decrease to 0 after convergence, e.g., KL values are in the order of e-1 on Office-home. This is also part of traits in the NPs family, suggesting the approximate posterior and the approximate prior encode different conditional information during the generation of latent variables.

\paragraph{Real-world examples or benchmarks for episodic multi-task learning.} Episodic multi-task learning has several potential applications in the real world, such as autonomous driving and robotic manipulations. In detail, the autonomous driving system needs to deal with different and related sensor data in an environment. However, the driving environment constantly changes along with the weather, time, country, etc. Thus, fast adapting of the current multi-tasker to new environments is challenging in this application and our method can be a plausible solution for this challenge.

\section{Properties of Valid Exchangeable Stochastic Processes}
\label{appendix:proof}

Here, we further demonstrate that HNPs are valid stochastic processes, as meeting the \textit{exchangeability} and \textit{marginalization} consistency conditions \cite{garnelo2018neural}. 
As stated in \cite{garnelo2018neural}: the conditions, including (finite) exchangeability and marginalization, are sufficient to define a stochastic process with the help of the Kolmogorov extension theorem. Here we follow the notations in the main paper. For convenience, we provide the used symbols and the corresponding descriptions in Table~\ref{tab: notaion}.

In this paper, we model the functional posterior distribution of the stochastic process by approximating the joint distribution over all target sets $p(\mathbf{y}^{1:M}_{\tau}| \mathbf{x}^{1:M}_{\tau}, \mathcal{C}^{1:M}_{\tau})$, which is conditioned on all context samples $C^{1:M}_{\tau}$. 
To distinguish the order set among different tasks, we use $n^m_{\tau}$ to denote the number of target samples corresponding to a specific task in the episode $\tau$. For simplicity, we omit the meta-knowledge $\omega$ and $\nu^{1:M}$ in the formulations during the proof. 

\begin{table}[ht]
\caption{Notations and their corresponding descriptions in this paper.}
\centering
\begin{adjustbox}{max width =1.0\linewidth}
\begin{tabular}{c|l}
\toprule
Notation & Description \\
\midrule
$(\cdot)_{\tau}$ & Variables correspond to an episode.  \\
$(\cdot)^{m}$ & Variables correspond to a single task.  \\
\midrule
$\mathcal{I}_{\tau}^{m}$ & A single task in the episode $\tau$, which is sampled from a specific task distribution.\\
$p(\mathcal{I}^m)$ & The specific task distribution.\\
$M$ & The number of task distributions and the number of tasks in a single episode. \\
$\mathcal{I}_{\tau}^{1:M}$ & All heterogeneous tasks in the episode $\tau$. \\
$\mathcal{C}$ & A context set.\\
$\mathcal{T}$ & A target set.\\
${\bar{x}}$ & A context sample feature of the context set. \\
${{x}}$ & A target sample feature of the target set. \\
$\mathbf{x}$ & Set of all target sample features in the target set. \\
${\bar{y}}$ & The ground truth of the context sample. \\
${{y}}$ & The ground truth of the target sample. \\
$\mathbf{y}$ &  Set of the ground truth of all target samples in the target set. \\
$\mathcal{N}_{\mathcal{C}}$ & The numbers of context samples in the context set.\\
$\mathcal{N}_{\mathcal{T}}$ & The numbers of target samples in the target set.\\
$\mathbf{z}^m_{\tau}$ & The introduced global latent representation for a given task.\\
$\mathbf{w}^m_{\tau, 1:O}$ & The introduced local latent parameters for a given task.\\
\bottomrule
\end{tabular}
\end{adjustbox}
\label{tab: notaion}
\end{table}

\subsection{Proof of Exchangeability Consistency}
We now provide the proof of \textit{Exchangeability Consistency}: the joint prediction distribution is invariant to the permutation of the given multiple tasks and the corresponding samples in each task. 

\begin{theorem} (Exchangability)
For finite ${n^*_{\tau}}=\sum^M_{m=1}n^m_{\tau}$, if ${\pi}^*_{\tau} = \{\pi^m_{\tau}\}_{m=1}^{M}$ is a permutation of $\{1, \cdots, {n^*}\}$ where $\pi^m_{\tau}$ is a permutation of the corresponding order set $\{1, \cdots, n^m_{\tau}\}$, then:
\begin{equation*}
p({\pi^*_{\tau}}(\mathbf{y}^{1:M}_{\tau})|{\pi^*_{\tau}}(\mathbf{x}^{1:M}_{\tau}), \mathcal{C}^{1:M}_{\tau})  = p(\mathbf{y}^{1:M}_{\tau} |\mathbf{x}^{1:M}_{\tau}, \mathcal{C}^{1:M}_{\tau}),
\end{equation*}
where ${\pi^*_{\tau}}(\mathbf{y}^{1:M}_{\tau}) := (\pi^1_{\tau}(\mathbf{y}^1_{\tau}), \cdots, \pi^M_{\tau}(\mathbf{y}^M_{\tau})) = ({y}_{\tau, \pi^*_{\tau}(1)}, \cdots, {y}_{\tau, \pi^*_{\tau}({n^*_{\tau}})})$ and ${\pi^*_{\tau}}(\{\mathbf{x}^{1:M}_{\tau}\}) := (\pi^1_{\tau}(\mathbf{x}^1_{\tau}), \cdots, \pi^M_{\tau}(\mathbf{x}^M_{\tau})) = ({x}_{\tau, {\pi^*_{\tau}}(1)}, \cdots, {x}_{\tau, {\pi^*_{\tau}}({n^{*}_{\tau}})})$.
\end{theorem}
\begin{proof}
\begin{small}
\begin{equation*}
\begin{aligned}
&p({\pi^*_{\tau}}(\mathbf{y}^{1:M}_{\tau})|{\pi^*_{\tau}}(\mathbf{x}^{1:M}_{\tau}), \mathcal{C}^{1:M}_{\tau}) \\
& = \int \int p( {\pi^*_{\tau}}(\mathbf{y}^{1:M}_{\tau})|{\pi^*_{\tau}}(\mathbf{x}^{1:M}_{\tau}), \mathbf{w}^{1:M}_{\tau, 1:O}) \Big (\prod_{m=1}^{M} p(\mathbf{w}^m_{\tau, 1:O} | \mathbf{z}^{m}_{\tau}, \mathcal{C}^{1:M}_{\tau}) \Big ) \Big(\prod_{m=1}^{M} p(\mathbf{z}^m_{\tau}|\mathcal{C}^m_{\tau}) \Big) d\mathbf{w}^{1:M}_{\tau, 1:O} d\mathbf{z}^{1:M}_{\tau} \\
& = \int \int \Big( \prod_{i=1}^{{n^*_{\tau}}} p( {y}_{\tau, {\pi^*_{\tau}}(i)} | {x}_{\tau, {\pi^*_{\tau}}(i)}, \mathbf{w}^{1:M}_{\tau, 1:O})\Big) \Big(\prod_{m=1}^{M} p(\mathbf{w}^m_{\tau, 1:O} | \mathbf{z}^{m}_{\tau}, \mathcal{C}^{1:M}_{\tau})\Big) \Big(\prod_{m=1}^{M} p(\mathbf{z}^m_{\tau}|\mathcal{C}^m_{\tau}) \Big) d\mathbf{w}^{1:M}_{\tau, 1:O} d\mathbf{z}^{1:M}_{\tau} \\
& = \int \int \prod_{m=1}^{M} \Big \{ \prod_{j=1}^{n^m_{\tau}} p({y}^m_{\tau, {\pi^m_{\tau}}(j)} | {x}^m_{\tau, {\pi^m_{\tau}}(j)}, \mathbf{w}^{m}_{\tau, 1:O}) p(\mathbf{w}^m_{\tau, 1:O} | \mathbf{z}^{m}_{\tau}, \mathcal{C}^{1:M}_{\tau})  p(\mathbf{z}^m_{\tau}|\mathcal{C}^m_{\tau}) \Big \} d\mathbf{w}^{1:M}_{\tau, 1:O} d\mathbf{z}^{1:M}_{\tau} \\
& = \int \int \prod_{m=1}^{M} \Big \{ \prod_{j=1}^{n^m_{\tau}} p({y}^m_{\tau, j} | {x}^m_{\tau, j}, \mathbf{w}^{m}_{\tau, 1:O}) p(\mathbf{w}^m_{\tau, 1:O} | \mathbf{z}^{m}_{\tau}, \mathcal{C}^{1:M}_{\tau})  p(\mathbf{z}^m_{\tau}|\mathcal{C}^m_{\tau}) \Big \} d\mathbf{w}^{1:M}_{\tau, 1:O} d\mathbf{z}^{1:M}_{\tau} \\
& = p(\mathbf{y}^{1:M}_{\tau} |\mathbf{x}^{1:M}_{\tau}, \mathcal{C}^{1:M}_{\tau}). \\
\end{aligned}
\end{equation*}
\end{small}
\end{proof}

\subsection{Proof of Marginalization Consistency}
We now aim to prove the \textit{Marginalization Consistency} of the proposed model: if marginalizing out a part of the target set in each task, the marginal distribution remains the same as defined on the original target sets without the marginalized part. 

\begin{theorem}(Marginalization)
Given ${\hat{n}^*_{\tau}} = \sum_{m=1}^{M} \hat{n}^m_{\tau}$, where $ 1 \le {\hat{n}^*_{\tau}} \le {n^*_{\tau}}$ and for each task $ 1\le \hat{n}^m_{\tau} \le {n}^m_{\tau}$, the consistency is:
\begin{equation*}
\int p(\mathbf{y}^{1:M}_{\tau} |\mathbf{x}^{1:M}_{\tau}, \mathcal{C}^{1:M}_{\tau}) d (\mathbf{y}^{1:M}_{\tau})_{{\hat{n}^*_{\tau}}+1:{{n}^*_{\tau}}}  = p((\mathbf{y}^{1:M}_{\tau})_{1:{\hat{n}^*_{\tau}}}|(\mathbf{x}^{1:M}_{\tau})_{1:{\hat{n}^*_{\tau}}}, \mathcal{C}^{1:M}_{\tau}),
\end{equation*}
where $(\mathbf{y}^{1:M}_{\tau})_{1:{\hat{n}^*_{\tau}}} = ((\mathbf{y}^{1}_{\tau})_{1:{\hat{n}^1_{\tau}}},
\cdots, (\mathbf{y}^{M}_{\tau})_{1:{\hat{n}^M_{\tau}}}) = (y_{\tau, 1}, \cdots, y_{\tau, \hat{n}^*_{\tau}})$ and $(\mathbf{x}^{1:M}_{\tau})_{1:{\hat{n}^*_{\tau}}} = ((\mathbf{x}^{1}_{\tau})_{1:{\hat{n}^1_{\tau}}},
\cdots, (\mathbf{x}^{M}_{\tau})_{1:{\hat{n}^M_{\tau}}})= (x_{\tau, 1}, \cdots, x_{\tau, \hat{n}^*_{\tau}})$.
\end{theorem}

\begin{proof}
\begin{small}
\begin{equation*}
\begin{aligned}
& \int p(\mathbf{y}^{1:M}_{\tau} |\mathbf{x}^{1:M}_{\tau}, \mathcal{C}^{1:M}_{\tau}) d (\mathbf{y}^{1:M}_{\tau})_{{\hat{n}^*_{\tau}}+1:{{n}^*_{\tau}}} \\
& =\int \int \int p(\mathbf{y}^{1:M}_{\tau}|\mathbf{x}^{1:M}_{\tau}, \mathbf{w}^{1:M}_{\tau, 1:O}) \Big(\prod_{m=1}^{M} p(\mathbf{w}^m_{\tau, 1:O} | \mathbf{z}^{m}_{\tau}, \mathcal{C}^{1:M}_{\tau})\Big) \Big(\prod_{m=1}^{M} p(\mathbf{z}^m_{\tau}|\mathcal{C}^m_{\tau})\Big) d\mathbf{w}^{1:M}_{\tau, 1:O} d\mathbf{z}^{1:M}_{\tau} d (\mathbf{y}^{1:M}_{\tau})_{{\hat{n}^*_{\tau}}+1:{{n}^*_{\tau}}}\\
& =\int \int \int \Big( \prod_{i=1}^{n^{*}_{\tau}} p(y_{\tau, i}| x_{\tau, i}, \mathbf{w}^{1:M}_{\tau, 1:O}) \Big) \Big(\prod_{m=1}^{M} p(\mathbf{w}^m_{\tau, 1:O} | \mathbf{z}^{m}_{\tau}, \mathcal{C}^{1:M}_{\tau})\Big) \Big(\prod_{m=1}^{M} p(\mathbf{z}^m_{\tau}|\mathcal{C}^m_{\tau})\Big) d\mathbf{w}^{1:M}_{\tau, 1:O} d\mathbf{z}^{1:M}_{\tau} d (\mathbf{y}^{1:M}_{\tau})_{{\hat{n}^*_{\tau}}+1:{{n}^*_{\tau}}}\\
& =\int \int \prod_{m=1}^{M} \Big \{ \int \prod_{j=1}^{n_{\tau}^{m}} p(y^m_{\tau, j}| x^m_{\tau, j}, \mathbf{w}^{m}_{\tau, 1:O}) p(\mathbf{w}^m_{\tau, 1:O} | \mathbf{z}^{m}_{\tau}, \mathcal{C}^{1:M}_{\tau}) p(\mathbf{z}^m_{\tau}|\mathcal{C}^m_{\tau}) d (\mathbf{y}^{m}_{\tau})_{{\hat{n}^m_{\tau}}+1:{{n}^m_{\tau}}}  \Big\} d\mathbf{w}^{1:M}_{\tau, 1:O} d\mathbf{z}^{1:M}_{\tau}\\
& =\int \int \prod_{m=1}^{M} \Big \{ \prod_{j=1}^{\hat{n}_{\tau}^{m}} p(y^m_{\tau, j}| x^m_{\tau, j}, \mathbf{w}^{m}_{\tau, 1:O}) p(\mathbf{w}^m_{\tau, 1:O} | \mathbf{z}^{m}_{\tau}, \mathcal{C}^{1:M}_{\tau}) p(\mathbf{z}^m_{\tau}|\mathcal{C}^m_{\tau}) \\
& ~~~~~~~~~~~~~~~~~~~~~~~~~~~~~~~~~~\int \prod_{j=\hat{n}_{\tau}^{m} +1}^{n_{\tau}^{m}} p(y^m_{\tau, j}| x^m_{\tau, j}, \mathbf{w}^{m}_{\tau, 1:O}) d (\mathbf{y}^{m}_{\tau})_{{\hat{n}^m_{\tau}}+1:{{n}^m_{\tau}}} \Big\} d\mathbf{w}^{1:M}_{\tau, 1:O} d\mathbf{z}^{1:M}_{\tau}\\
& =\int \int \prod_{m=1}^{M} \Big \{ \prod_{j=1}^{\hat{n}_{\tau}^{m}} p(y^m_{\tau, j}| x^m_{\tau, j}, \mathbf{w}^{m}_{\tau, 1:O}) p(\mathbf{w}^m_{\tau, 1:O} | \mathbf{z}^{m}_{\tau}, \mathcal{C}^{1:M}_{\tau}) p(\mathbf{z}^m_{\tau}|\mathcal{C}^m_{\tau}) \Big\} d\mathbf{w}^{1:M}_{\tau, 1:O} d\mathbf{z}^{1:M}_{\tau}\\
& = p((\mathbf{y}^{1:M}_{\tau})_{1:{\hat{n}^*_{\tau}}}|(\mathbf{x}^{1:M}_{\tau})_{1:{\hat{n}^*_{\tau}}}, \mathcal{C}^{1:M}_{\tau}).
\end{aligned}
\end{equation*}
\end{small}
\end{proof}

\section{Tractable and Scalable Optimization}
\label{sec: appendix/optimization}

For the proposed HNPs, it is intractable to obtain the true joint posterior $p(\mathbf{w}_{\tau, 1:O}^{1:M}, \mathbf{z}_{\tau}^{1:M}|\mathcal{T}_{\tau}^{1:M}; \omega, \nu^{1:M})$ for each episode. Thus, we employ variational inference to optimize the designed model by approximating the true joint posterior in each episode. To do so, we introduce the variational joint posterior distribution:
\begin{equation}
\begin{aligned}
& q_{\theta, \phi}(\mathbf{w}_{\tau, 1:O}^{1:M}, \mathbf{z}_{\tau}^{1:M}|\mathcal{T}_{\tau}^{1:M}; \omega, \nu^{1:M}) = \prod_{m=1}^M q_{\theta}(\mathbf{z}_{\tau}^m|\mathcal{T}_{\tau}^{m}; \nu^m) q_{\phi}( \mathbf{w}^m_{\tau, 1:O} | \mathbf{z}_{\tau}^m, \mathcal{T}_{\tau}^{1:M}; \omega), \\
\end{aligned}
\label{eq: variational_posteriors}
\end{equation}
where $q_{\theta}(\mathbf{z}_{\tau}^m|\mathcal{T}_{\tau}^{m}; \nu^m)$ and $q_{\phi}(\mathbf{w}^m_{\tau, 1:O} | \mathbf{z}_{\tau}^m, \mathcal{T}_{\tau}^{1:M}; \omega)$ are variational posteriors of their corresponding latent variables. We note that variational posteriors are inferred from the target sets that are available in the meta-training stage. The variational posteriors are parameterized as diagonal Gaussian distributions~\cite{kingma2013auto}. The inference networks $\theta$ and $\phi$ are shared by the variational posteriors and their corresponding priors, following the protocol of vanilla NPs~\cite{garnelo2018neural}.

\subsection{Derivation of Approximate ELBO for HNPs}
By incorporating the variational posteriors in Eq.~(\ref{eq: variational_posteriors}) into the modeling of HNPs in the main paper, we can derive the approximate ELBO $L_{\rm{HNPs}}(\omega, \nu^{1:M}, \theta, \phi)$ as follows:
\begin{small}
\begin{equation}
\begin{aligned}
& \log p(\mathcal{T}_{\tau}^{1:M} |  \mathcal{C}_{\tau}^{1:M};  \omega, \nu^{1:M}) \\
& = \sum_{m=1}^{M} \log p(\mathcal{T}_{\tau}^{m} |  \mathcal{C}_{\tau}^{1:M};  \omega, \nu^{m}) \\
& = \sum_{m=1}^{M} \bigg \{ \log  \int  \Big \{ \int p(\mathbf{y}_{\tau}^{m}|\mathbf{x}_{\tau}^{m}, \mathbf{w}_{\tau, 1:O}^m) p_{\phi}( \mathbf{w}_{\tau, 1:O}^m | \mathbf{z}_{\tau}^m, \mathcal{C}_{\tau}^{1:M}; \omega) d \mathbf{w}^m_{\tau, 1:O}   \Big \} p_{\theta}(\mathbf{z}_{\tau}^m|\mathcal{C}_{\tau}^{m}; \nu^m) d\textbf{z}_{\tau}^m \bigg \} \\
& =  \sum_{m=1}^{M} \bigg \{  \log \int \Big\{ \int p(\mathbf{y}_{\tau}^{m}|\mathbf{x}_{\tau}^{m}, \mathbf{w}_{\tau, 1:O}^m) p_{\phi}( \mathbf{w}_{\tau, 1:O}^m | \mathbf{z}_{\tau}^m, \mathcal{C}_{\tau}^{1:M}; \omega) \frac{q_{\phi}( \mathbf{w}^m_{\tau, 1:O} | \mathbf{z}_{\tau}^m, \mathcal{T}_{\tau}^{1:M}; \omega)}{ q_{\phi}( \mathbf{w}^m_{\tau, 1:O} | \mathbf{z}_{\tau}^m, \mathcal{T}_{\tau}^{1:M}; \omega)} d\mathbf{w}^m_{\tau, 1:O} \Big\} \\
& ~~~~~~~~~~~~~~~~ p_{\theta}(\mathbf{z}_{\tau}^m|\mathcal{C}_{\tau}^{m}; \nu^m) \frac{q_{\theta}(\mathbf{z}_{\tau}^m|\mathcal{T}_{\tau}^{m}; \nu^m)}{q_{\theta}(\mathbf{z}_{\tau}^m|\mathcal{T}_{\tau}^{m}; \nu^m)} d\textbf{z}_{\tau}^m \bigg \} \\
& \ge  \sum_{m=1}^{M} \bigg \{  \mathbb{E}_{q_{\theta}(\mathbf{z}_{\tau}^m|\mathcal{T}_{\tau}^{m}; \nu^m)} \Big\{ \log \int p(\mathbf{y}_{\tau}^{m}|\mathbf{x}_{\tau}^{m}, \mathbf{w}_{\tau, 1:O}^m) p_{\phi}( \mathbf{w}_{\tau, 1:O}^m | \mathbf{z}_{\tau}^m, \mathcal{C}_{\tau}^{1:M}; \omega) \frac{q_{\phi}( \mathbf{w}^m_{\tau, 1:O} | \mathbf{z}_{\tau}^m, \mathcal{T}_{\tau}^{1:M}; \omega)}{ q_{\phi}( \mathbf{w}^m_{\tau, 1:O} | \mathbf{z}_{\tau}^m, \mathcal{T}_{\tau}^{1:M}; \omega)} d\mathbf{w}^m_{\tau, 1:O} \Big\} \\
& ~~~~~~~~~~~~~~~~ - \mathbb{D}_{\rm{KL}} [q_{\theta}(\mathbf{z}_{\tau}^m|\mathcal{T}_{\tau}^{m}; \nu^m) || p_{\theta}(\mathbf{z}_{\tau}^m|\mathcal{C}_{\tau}^{m}; \nu^m)] \bigg \} \\
& \ge \sum_{m=1}^{M} \bigg \{ \mathbb{E}_{q_{\theta}(\mathbf{z}_{\tau}^m|\mathcal{T}_{\tau}^{m}; \nu^m)}  \Big \{ \mathbb{E}_{q_{\phi}(\mathbf{w}^m_{\tau, 1:O} | \mathbf{z}_{\tau}^m, \mathcal{T}_{\tau}^{1:M}; \omega)}[\log p(\mathbf{y}_{\tau}^{m}|\mathbf{x}_{\tau}^{m}, \mathbf{w}^{m}_{\tau, 1:O})] \\
& ~~~~~~~~~~~~~~~~ - \mathbb{D}_{\rm{KL}}[q_{\phi}(\mathbf{w}^m_{\tau, 1:O} | \mathbf{z}_{\tau}^m, \mathcal{T}_{\tau}^{1:M}; \omega) || p_{\phi}(\mathbf{w}^m_{\tau, 1:O} | \mathbf{z}_{\tau}^m, \mathcal{C}_{\tau}^{1:M}; \omega)] \Big \} \\
& ~~~~~~~~~~~~~~~~ - \mathbb{D}_{\rm{KL}}[q_{\theta}(\mathbf{z}_{\tau}^m|\mathcal{T}_{\tau}^{m}; \nu^m)||p_{\theta}(\mathbf{z}_{\tau}^m|\mathcal{C}_{\tau}^{m}; \nu^m)]  \bigg \} = L_{\rm{HNPs}}(\omega, \nu^{1:M}, \theta, \phi).
\end{aligned}
\label{eq: derivation_eblo}
\end{equation}
\end{small}

In general, when the variational joint posterior is flexible enough, the posterior approximation gap between the variational posterior and the intractable true posterior can be reduced to an arbitrarily small quantity~\cite{wang2023bridge}. In this case, maximizing the approximate ELBO increases the overall log likelihood in the proposed model accordingly. We construct task-specific decoders in a data-driven way by inferring local latent parameters $\mathbf{w}^m_{1:O}$ from all context sets and the meta-knowledge. This enables our model to amortize the training cost of each task-specific decoder, further reducing the model's over-fitting behaviors for episodic multi-task learning.

\subsection{Meta-Training Objective}
\vspace{-2mm}
In practice, we consider the loss function as the negative approximate ELBO of HNPs as given in Eq.~(\ref{eq: derivation_eblo}). By adopting Monte Carlo sampling~\cite{kingma2013auto, kingma2015variational}, the meta-training objective for each episode is:  
\begin{small}
\begin{equation}
\begin{aligned}
& - L_{\rm{HNPs}}(\omega, \nu^{1:M}, \theta, \phi) \approx \sum_{m=1}^{M} \bigg \{ \frac{1}{N_z} \sum_{i=1}^{N_z}  \Big \{ \frac{1}{N_w} \sum_{j=1}^{N_w} [ - \log p(\mathbf{y}_{\tau}^{m}|\mathbf{x}_{\tau}^{m}, {\mathbf{w}^{m}_{\tau, 1:O}}^{(j)})] \\
& + \mathbb{D}_{\rm{KL}}[q_{\phi}(\mathbf{w}^m_{\tau, 1:O} | {\mathbf{z}_{\tau}^m}^{(i)}, \mathcal{T}_{\tau}^{1:M}; \omega) || p_{\phi}(\mathbf{w}^m_{\tau, 1:O} | {\mathbf{z}_{\tau}^m}^{(i)}, \mathcal{C}_{\tau}^{1:M}; \omega)] \Big \} \\
& + \mathbb{D}_{\rm{KL}}[q_{\theta}(\mathbf{z}_{\tau}^m|\mathcal{T}_{\tau}^{m}; \nu^m)||p_{\theta}(\mathbf{z}_{\tau}^m|\mathcal{C}_{\tau}^{m}; \nu^m)]  \bigg \},\\
\end{aligned}
\label{eq:empirical objective}
\end{equation}
\end{small}where ${\mathbf{z}_{\tau}^m}^{(i)}$ and ${\mathbf{w}^{m}_{\tau, 1:O}}^{(j)}$ are sampled from their corresponding variational posteriors.
${N_z}$ and $N_w$ are the number of Monte Carlo samples for ${\mathbf{z}_{\tau}^m}$ and ${\mathbf{w}^{m}_{\tau, 1:O}}$, respectively.

\subsection{Meta-Test Prediction}
\vspace{-2mm}
At the meta-test stage, we perform predictions with the learned model on the target sets for a new episode ${\tau^*}$, which involves the prior distributions of global and local latent variables. We again approximate the predictive distribution with Monte Carlo estimates:
\begin{small}
\begin{equation}
\begin{aligned}
& p  (\mathcal{T}_{\tau^*}^{1:M} | \mathcal{C}_{\tau^*}^{1:M}; \omega, \nu^{1:M}) \approx \prod_{m=1}^{M} \bigg \{ \frac{1}{N_z} \sum_{i=1}^{N_z} \frac{1}{N_w}\sum_{j=1}^{N_w} p(\mathbf{y}_{\tau^*}^{m}|\mathbf{x}_{\tau^*}^{m}, {\mathbf{w}_{\tau^*, 1:O}^{m}}^{(j)}) \bigg \},
\\
\end{aligned}
\label{eq:predictive distribution}
\end{equation}
\end{small}where ${\mathbf{w}_{\tau^*, 1:O}^{m}}^{(j)} \sim p_{\phi}(\mathbf{w}^m_{\tau^*, 1:O} | {\mathbf{z}_{\tau^*}^m}^{(i)}, \mathcal{C}_{\tau^*}^{1:M}; \omega)$ and ${\mathbf{z}_{\tau^*}^m}^{(i)} \sim p_{\theta}(\mathbf{z}_{\tau^*}^m|\mathcal{C}_{\tau^*}^{m}; \nu^m)$. Here the Monte Carlo samples follow their corresponding prior distributions since the target sets are unavailable during the meta-test. 

\section{More Experimental Details}
\label{appendix:implementation_details}

\subsection{Transformer-structured Inference Modules in Regression Scenarios}
Here we present the implementation details of transformer-structured inference module $\theta$ in regression scenarios. 
In a single episode $\tau$, the module $\theta$ encodes the task-specific information into each refined task-specific token $\nu^m$, and then infers the prior distribution or the variational posterior distribution for the global latent representation $\mathbf{z}^m_{\tau}$. For episodic multi-task regression, the local latent parameters $\mathbf{w}_{\tau, 1:O}^{m}$ construct a task-specific regressor during inference. We assume that the output of the decoder follows a Gaussian distribution for regression tasks. Thus, $\mathbf{w}^m_{\tau, 1:O}$ are instantiated as parameters for generating the mean and variance of predictions.

\subsection{Backbone and Training Details}
Following the protocol of \cite{long2017learning}, we apply the pre-trained deep model as the backbone for the proposed method and baselines to extract the input features under the episodic multi-task classification setup. To be specific, we adopt VGGnet~\cite{simonyan2014very} for \texttt{Office-Home}, and its feature size is $4096$. We take ResNet18~\cite{he2016deep} for \texttt{DomainNet} with input size $512$. In practice, we train our method and baselines by the Adam optimizer~\cite{kingma2014adam} using an NVIDIA Tesla V100 GPU. The learning rate is initially set as $1e-4$ and decreases with a factor of $0.5$ every $3,000$ iterations.

\section{Algorithm of the proposed HNPs}

\begin{algorithm}[ht]
\SetAlgoLined
\SetKwInOut{Input}{Input}
\SetKwInOut{Output}{Output}
\Input{$M$ distinct and relevant task distributions $p(\mathcal{I}^{1:M})$, 
numbers of Monte Carlo samples $N_z$ and $N_w$, learning rates $\alpha$.}
\Output{Meta-trained transformer inference module $\theta$ and $\phi$, learnable tokens $\omega_{1:O}$ and $\nu^{1:M}$.}

Initialize transformer inference module $\{\theta, \nu^{1:M}\}$ \;

Initialize transformer inference module $\{\phi, \omega_{1:O}\}$ \;


\While{\text{Meta-Training not Completed}}{
\texttt{// sample a meta-training episode indexed with $\tau$.}

\For{$m=1$ to $M$}{
Sample a task $\mathcal{I}^m_{\tau} \sim p(\mathcal{I}^m)$, which shares the same target space $\mathcal{Y}_{\tau}$ with other tasks in the episode\;

Sample a context set $\mathcal{C}^m_{\tau}$ and a target set $\mathcal{T}^m_{\tau}$ for the task $\mathcal{I}^m_{\tau}$\;
}
\texttt{// infer hierarchical latent variables.}

\For{$m=1$ to $M$}{
Apply transformer inference module $\theta$ to infer prior and variational posterior distributions over $\mathbf{z}^m_{\tau}$ as Eq.~(5), Eq.~(6), and Eq.~(7) in the main paper\;

Draw $N_z$ samples from the variational posterior, $\{{\mathbf{z}^m_{\tau}}^{(i)}\}_{i=1}^{N_z}$\;

\For{$i=1$ to $N_z$}{
\For{$o=1$ to $O$}{
Apply transformer inference module $\phi$ to infer prior and variational posterior distributions over $\mathbf{w}^m_{\tau, o}$ as Eq.~(8), Eq.~(9), and Eq.~(10) in the main paper\;

Draw $N_w$ samples from the variational posterior, $\{{\mathbf{w}^m_{\tau, o}}^{(j)}\}_{j=1}^{N_w}$\;
}
}
}

\texttt{// optimize the objective.}

Compute predictive distributions and minimize the empirical objective in Eq.~(\ref{eq:empirical objective}) \;

Update $\theta$, $\phi$, $\omega_{1:O}$ and $\nu^{1:M}$ with learning rate $\alpha$.
}
\caption{Meta-training phase of HNPs.}
\label{algorithm}
\end{algorithm}

\begin{algorithm}[ht]
\SetAlgoLined
\SetKwInOut{Input}{Input}
\SetKwInOut{Output}{Output}
\Input{Meta-trained transformer inference module $\theta$ and $\phi$, learned tokens $\omega_{1:O}$ and $\nu^{1:M}$, numbers of Monte Carlo samples $N_z$ and $N_w$.}
\Output{Prediction results.}

\While{\text{Meta-Test not Completed}}{
\texttt{// given a meta-test episode indexed with $\tau^*$.}

Collect the context sets of all tasks in the episode, $\mathcal{C}^{1:M}_{\tau^{*}}$. 

\For{$m=1$ to $M$}{
Apply transformer inference module $\theta$ to infer prior distributions over $\mathbf{z}^m_{\tau^{*}}$ as Eq.~(5), Eq.~(6), and Eq.~(7) in the main paper\;

Draw $N_z$ samples from the prior distribution, $\{{\mathbf{z}^m_{\tau^{*}}}^{(i)}\}_{i=1}^{N_z}$\;

\For{$i=1$ to $N_z$}{
\For{$o=1$ to $O$}{
Apply transformer inference module $\phi$ to infer the prior distribution over $\mathbf{w}^m_{\tau^{*}, o}$ as Eq.~(8), Eq.~(9), and Eq.~(10) in the main paper\;

Draw $N_w$ samples from the prior distribution, $\{{\mathbf{w}^m_{\tau^{*}, o}}^{(j)}\}_{j=1}^{N_w}$\;
}
}
}

Perform predictions on the target sets in Eq.~(\ref{eq:predictive distribution}) \;

}
\caption{Meta-test phase of HNPs.}
\label{algorithm: test}
\end{algorithm}

\begin{table}[ht]
\centering
\caption{\textbf{Performance comparisons on Office-Home under the Xtask5way1shot setup during inference.}}
\begin{tabular}{lcccc}
\hline
Number of tasks & 1 & 2 & 3 & 4 \\
\hline
Average accuracy & 64.33 $\pm$ 0.85 & 70.75 $\pm$ 0.67 & 74.95 $\pm$ 0.57 & 76.29 $\pm$ 0.51 \\
\hline
\end{tabular}
\label{tab: rebuttal7}
\end{table}

\begin{table}[ht]
\centering
\caption{\textbf{Performance comparisons under the 4task20way1shot and 20way4shot setups.}}
\begin{tabular}{ccc}
\toprule
Methods & 4task20way1shot & 20way4shot \\
\midrule
MAML & $34.29 \pm 0.19$ & $37.23 \pm 0.25$ \\
Proto.Net & $32.72 \pm 0.18$ & $37.12 \pm 0.22$ \\
Ours & $51.82 \pm 0.23$ & - \\
\bottomrule
\end{tabular}
\label{tab: rebuttal3}
\end{table}

\begin{table}[ht]
\centering
\caption{\textbf{Performance comparisons with more recent baselines.}}
\begin{tabular}{lcc}
\toprule
Methods & 1-shot & 5-shot \\
\midrule
MTL-bridge~\cite{wang2021bridging} & $64.31 \pm 0.55$ & $75.10 \pm 0.51$ \\
MLTI~\cite{yao2021meta} & $70.69 \pm 0.73$ & $79.59 \pm 0.58$ \\
Ours & $76.29 \pm 0.51$ & $80.80 \pm 0.42$ \\
\bottomrule
\end{tabular}
\label{tab: rebuttal4}
\end{table}

\begin{table}[ht]
\centering
\caption{\textbf{Performance comparisons on the Office31 dataset.}}
\begin{adjustbox}{max width =\textwidth}
\begin{tabular}{lcccccc}
\hline
Methods & ERM & Proto.Net & CNPs & NPs & TNP-D & Ours \\
\toprule
Average Accuracy & $63.53 \pm 0.71$ & $64.54 \pm 0.64$ & $49.02 \pm 0.74$ & $40.52 \pm 0.75$ & $69.69 \pm 0.87$ & $71.89 \pm 0.52$ \\
\bottomrule
\end{tabular}
\end{adjustbox}
\label{tab: rebuttal5}
\end{table}

\section{More Experimental Results under Episodic Multi-Task Setup}

\subsection{Effects of More “Tasks”}
To investigate the effects of more "tasks" in the episodic multi-task setup, we perform experiments on Office-Home under the Xway5way1shot setup by gradually increasing tasks during inference.
Table~\ref{tab: rebuttal7} shows that the average accuracy increases with more tasks. The main reason is that more tasks can provide richer transferable information. Our model benefits from the positive transfer among tasks, and thus obtaining higher performance gain from more tasks.

\subsection{4task20way1shot v.s. 20way4shot}
To show the effectiveness of the proposed method, we conduct experiments with the 20-way 4-shot setup, which needs to mix samples from all tasks in one episode. As shown in Table~\ref{tab: rebuttal3},  MAML and Proto.Net perform better under the 20-way 4-shot but cannot outperform our method. The main reason is that our method can better handle distribution shifts among tasks by exploring task-relatedness rather than simply mixing them together.

\subsection{Comparisons with More Recent Works}
To compare the proposed method with more recent works, we perform experiments on Office-Home under the 4task5way1shot and 4task5way5shot setups.
We provide some brief descriptions of two recent works as follows:

~\cite{wang2021bridging} theoretically addresses the conclusion that MTL methods are powerful and efficient alternatives to GBML for meta-learning applications. However, our method inherits the advantages of multi-task learning and meta-learning: simultaneously generalizing meta-knowledge from past to new episodes and exploiting task-relatedness across heterogeneous tasks in every single episode. Thus, our method is more suitable for solving the data-insufficiency problem.

~\cite{yao2021meta} augments the task set in meta-learning through interpolation. Our method fully utilizes several observed tasks in a single episode rather than generating additional tasks.

Table~\ref{tab: rebuttal4} shows that our method significantly outperforms other baselines with severely insufficient data, such as 1-shot. This is consistent with the conclusion obtained in the main paper.

\subsection{Comparisons on Another Benchmark Dataset}
We validate the performance of methods on the Office31 dataset~\cite{saenko2010adapting, kulis2011you} under the 3task5way1shot setup. This dataset contains 31 object categories in three domains: Amazon, DSLR, and Webcam. Table~\ref{tab: rebuttal5} shows our method outperforms baseline methods, demonstrating our model's effectiveness in addressing the data insufficiency under the episodic setup. 

\section{More Experimental Results under Conventional Multi-Task Setup}
\label{appendix:more_results_MIMO}
To show comparisons with existing multi-task models, which are designed for conventional multi-task learning, we extend the proposed HNPs to conventional multi-task learning settings for both regression and classification tasks by considering only one episode during training and inference. 
Under conventional multi-task settings (MIMO), we investigate their effectiveness in exploring shared knowledge when limited tasks and samples are available during training.

\subsection{Conventional Multi-Task Regression}

\paragraph{Dataset and Settings.}
We show the effectiveness of HNPs for conventional multi-task regression. We design experiments for rotation angle estimation on the $\texttt{Rotated MNIST}$ dataset~\cite{lecun1998gradient}. Each task is an angle estimation problem for a digit, and different tasks corresponding to individual digits are related because they share the same rotation angle space. Each image is rotated by $0^\circ$ through $90^\circ$ in intervals of $10^\circ$, where the rotation angle is the regression target. We randomly choose $0.1\%$ training samples per task per angle as the training set. 

We use the average of normalized mean squared errors (NMSE) of all tasks as the performance measurement. The lower NMSE, the better the performance. We provide $95$\% confidence intervals for the errors from five runs. Descriptions of baselines can be found in Section~\ref{appendix: subsection_MIMO_class}.

\paragraph{Results and Discussions.}
The experimental results are summarized in Table~\ref{table_mnist}. 
The proposed HNPs outperform other counterpart methods by yielding a lower mean error. This demonstrates the effectiveness of HNPs in capturing task-relatedness for improved regression performance.

\begin{table}[ht]
\begin{center}
\caption{\textbf{Conventional multi-task regression (normalized mean squared errors) for rotation angle estimation.} 
}
\label{table_mnist}
\centering
\begin{adjustbox}{max width =\textwidth}
\begin{tabular}{l|c>{\columncolor[gray]{.9}}c}
\toprule
Methods 
 &  Average NMSE\\
\midrule
STL
& .215 {\scriptsize$\pm$.001} \\
VSTL
& .224 {\scriptsize$\pm$.004} \\
\midrule
BMTL
& .118 {\scriptsize$\pm$.003}  \\
VBMTL
& .121 {\scriptsize$\pm$.003} 
\\
LearnToBranch~\cite{guo2020learning}
& {.109} {\scriptsize$\pm$.002}
\\
VMTL~\cite{shen2021variational}
& .110 {\scriptsize$\pm$.003}
\\
\rowcolor{LightGray}
HNPs
& \textbf{.103} {\scriptsize$\pm$.001} \\
\bottomrule
\end{tabular}
\end{adjustbox}
\end{center}
\end{table}

\subsection{Conventional Multi-Task Classification}
\label{appendix: subsection_MIMO_class}

\paragraph{Datasets and Settings.}
\texttt{Office-Caltech} \cite{gong2012geodesic} contains ten categories shared between Office-31 \cite{saenko2010adapting} and Caltech-256 \cite{griffin2007caltech}. One task uses data from Caltech-256, and the other tasks use data from Office-31, whose images were collected from three distinct domains/tasks, namely Amazon, Webcam and DSLR. There are $8 \sim 151$ samples per category per task and $2,533$ images. \texttt{ImageCLEF} \cite{long2017learning}, the benchmark for the ImageCLEF domain adaptation challenge, contains $12$ common categories shared by four public datasets/tasks: Caltech-256, ImageNet ILSVRC 2012, Pascal VOC 2012, and Bing. There are $2,400$ images in total. \texttt{Office-Home} \cite{venkateswara2017Deep} mentioned in the main paper is also used under this setting. 

We adopt the standard evaluation protocols~\cite{long2017learning} for multi-task classification datasets. We randomly select $5\%$, $10\%$ and $20\%$ labeled data for training, which correspond to about 3, 6 and 12 samples per category per task, respectively. In this case, each task has insufficient training data for building a reliable classifier without overfitting. The average accuracy of all tasks is used for measuring the overall performance. We again provide $95$\% confidence intervals for the errors from five runs. 

\paragraph{Alternatives Methods.}
We conduct a thorough comparison with alternative multi-task learning models. Single-task learning (STL) is implemented by task-specific feature extractors and predictors without knowledge sharing among tasks. Basic multi-task learning (BMTL) shares feature extractors and adds task-specific predictors. We also define variational extensions of single-task learning (VSTL) and basic multi-task learning (VBMTL), which treat predictors as latent variables~\cite{shen2021variational}. For a fair comparison, all the baseline methods mentioned above share the same architecture of the feature extractor and the train-test splits. 
We also compare the proposed HNPs to representative multi-task models.
MTL-Uncertainty~\cite{kendall2018multi}, MRN~\cite{long2017learning},  LearnToBranch~\cite{guo2020learning} are deep MTL methods, employing
deep neural networks to construct information-sharing mechanisms for tasks.
TCGBML~\cite{bakker2003task}, MTVIB~\cite{qian2020multi} and VMTL~\cite{shen2021variational} are probabilistic MTL methods, applying Bayes frameworks to model the relationships among tasks.

\begin{table*}
\caption{\textbf{Classification performance (average accuracy) on \texttt{Office-Home}, \texttt{Office-Caltech} and \texttt{ImageCLEF}.}}
\centering
\begin{adjustbox}{max width =\linewidth}
\begin{tabular}{l|ccc|ccc|ccc}
\toprule
& \multicolumn{3}{c|}{\texttt{Office-Home}}         & \multicolumn{3}{c|}{\texttt{Office-Caltech}}   & \multicolumn{3}{c}{\texttt{ImageCLEF}}  \\
\cline{2-10}
Methods 
&\texttt{5\%} &\texttt{10\%} &\texttt{20\%}  
&\texttt{5\%} &\texttt{10\%} &\texttt{20\%}  
&\texttt{5\%} &\texttt{10\%} &\texttt{20\%} \\
\midrule
STL
&49.2 {\scriptsize$\pm$0.2}
&58.3 {\scriptsize$\pm$0.1}
&64.9 {\scriptsize$\pm$0.1}
&88.6 {\scriptsize$\pm$0.3}
&90.7 {\scriptsize$\pm$0.2}
&92.4 {\scriptsize$\pm$0.3}
&62.6 {\scriptsize$\pm$0.2} 
&69.7 {\scriptsize$\pm$0.3} 
&76.2 {\scriptsize$\pm$0.3}
\\
VSTL
&51.1 {\scriptsize$\pm$0.1}
&60.2 {\scriptsize$\pm$0.2} 
&65.8 {\scriptsize$\pm$0.2}
&89.0 {\scriptsize$\pm$0.2}
&91.1 {\scriptsize$\pm$0.2}
&93.4 {\scriptsize$\pm$0.3}
&64.9 {\scriptsize$\pm$0.3}
&70.8 {\scriptsize$\pm$0.3}
&77.2 {\scriptsize$\pm$0.2}
\\
\midrule
BMTL
&50.4 {\scriptsize$\pm$0.1}
&59.5 {\scriptsize$\pm$0.1}
&65.6 {\scriptsize$\pm$0.1}
&89.5 {\scriptsize$\pm$0.3}
&92.3 {\scriptsize$\pm$0.2}
&93.1 {\scriptsize$\pm$0.1}
&65.7 {\scriptsize$\pm$0.4}
&72.0 {\scriptsize$\pm$0.3}
&76.8 {\scriptsize$\pm$0.3}
\\
VBMTL
&51.3 {\scriptsize$\pm$0.1}
&60.9 {\scriptsize$\pm$0.1}
&67.0 {\scriptsize$\pm$0.2}
&90.8 {\scriptsize$\pm$0.6}
&93.2 {\scriptsize$\pm$0.2}
&93.5 {\scriptsize$\pm$0.1}
&67.1 {\scriptsize$\pm$0.3}
&73.0 {\scriptsize$\pm$0.7}
&78.0 {\scriptsize$\pm$0.2}
\\
MTL-Uncertainty\cite{kendall2018multi} 
&51.8 {\scriptsize$\pm$0.1}
&57.2 {\scriptsize$\pm$0.2}
&66.8 {\scriptsize$\pm$0.2}
&91.2 {\scriptsize$\pm$0.3}
&93.8 {\scriptsize$\pm$0.2}
&94.7 {\scriptsize$\pm$0.3}
&74.6 {\scriptsize$\pm$0.2}
&76.9 {\scriptsize$\pm$0.3}
&79.2 {\scriptsize$\pm$0.3}
\\
MRN~\cite{long2017learning} 
&{57.4 {\scriptsize$\pm$0.1}}
&{63.4} {\scriptsize$\pm$0.2}
&67.1 {\scriptsize$\pm$0.1}
&93.4 {\scriptsize$\pm$0.2} 
&94.8 {\scriptsize$\pm$0.3}  
&95.1 {\scriptsize$\pm$0.1}
&73.7 {\scriptsize$\pm$0.4} 
&75.8 {\scriptsize$\pm$0.2} 
&79.7 {\scriptsize$\pm$0.3}
\\
LearnToBranch~\cite{guo2020learning} 
&38.3 {\scriptsize$\pm$0.5} 
&51.5 {\scriptsize$\pm$0.3}
&62.2 {\scriptsize$\pm$0.4}
&74.6 {\scriptsize$\pm$0.9}   
&80.4 {\scriptsize$\pm$1.2}  
&89.9 {\scriptsize$\pm$0.8}
&51.7 {\scriptsize$\pm$0.9} 
&62.6 {\scriptsize$\pm$0.8} 
&71.6 {\scriptsize$\pm$0.4} 
\\
TCGBML\cite{bakker2003task} 
&52.8 {\scriptsize$\pm$0.1}
&60.0 {\scriptsize$\pm$0.2}
&68.7 {\scriptsize$\pm$0.2}
&91.8 {\scriptsize$\pm$0.1}
&95.0 {\scriptsize$\pm$0.2}
&95.1 {\scriptsize$\pm$0.1}
&73.9 {\scriptsize$\pm$0.3}
&76.5 {\scriptsize$\pm$0.4}
&79.3 {\scriptsize$\pm$0.4}
\\
MTVIB~\cite{qian2020multi}
&49.9 {\scriptsize$\pm$0.2}
&55.3 {\scriptsize$\pm$0.1}
&66.2 {\scriptsize$\pm$0.1}
&91.1 {\scriptsize$\pm$0.3}
&94.1 {\scriptsize$\pm$0.3}
&95.0 {\scriptsize$\pm$0.2}
&74.0 {\scriptsize$\pm$0.4}
&77.3 {\scriptsize$\pm$0.3}
&78.9 {\scriptsize$\pm$0.5}
\\
VMTL~\cite{shen2021variational}
&58.3 {\scriptsize$\pm$0.1}
&65.0 {\scriptsize$\pm$0.0}
&69.2 {\scriptsize$\pm$0.2}
&93.8 {\scriptsize$\pm$0.1}
&{95.3} {\scriptsize$\pm$0.0}
&95.2 {\scriptsize$\pm$0.1}
&{76.2} {\scriptsize$\pm$0.3}
&77.9 {\scriptsize$\pm$0.2}
&80.2 {\scriptsize$\pm$0.1}
\\
\rowcolor{LightGray}
HNPs
&\textbf{60.0} {\scriptsize$\pm$0.1}
&\textbf{66.2} {\scriptsize$\pm$0.2}
&\textbf{70.9} {\scriptsize$\pm$0.2}
&\textbf{94.6} {\scriptsize$\pm$0.1}
&\textbf{95.4} {\scriptsize$\pm$0.1}
&\textbf{95.8} {\scriptsize$\pm$0.1}
&\textbf{76.4} {\scriptsize$\pm$0.1}
&\textbf{79.5} {\scriptsize$\pm$0.1}
&\textbf{80.9} {\scriptsize$\pm$0.1}
\\
\bottomrule
\end{tabular}
\end{adjustbox}
\label{tab:regular_MIMO}
\end{table*}

\paragraph{Results and Discussions.} We provide more comprehensive comparisons on $\texttt{Office-Home}$, \texttt{Office-Caltech} and \texttt{ImageCLEF} in Table~\ref{tab:regular_MIMO}. The best results are marked in bold. Our HNPs achieve competitive and even better performance on conventional multi-task classification datasets with different train-test splits. VSTL and VBMTL perform better than STL and BMTL, demonstrating the benefits of Bayes frameworks. Compared with multi-task probabilistic baselines, including VBMTL, TCGBML, MTVIB and VMTL, our HNPs can model more complex functional distributions with more powerful priors by inferring both representations and parameters for predictive functions. 

Compared with VMTL~\cite{shen2021variational}, which neglects the hierarchical architecture of latent variables, the proposed HNPs show better performance. This demonstrates that by modeling the complex dependencies between heterogeneous context sets within the hierarchical Bayes framework, HNPs explore task-relatedness better. The hierarchical Bayes framework enables our model to explore the relevant knowledge even in the presence of distribution shifts among tasks. 

\section{Application to Brain Image Segmentation}
\label{appendix:more_results_Brain}
 
To show that HNPs have the potential to be helpful in settings other than categorization and regression, we apply the proposed HNPs to brain image segmentation.

\paragraph{Dataset and Settings.}

We adopt a brain image dataset~\cite{buda2019association} with lower-grade gliomas collected from $110$ patients. The number of images varies among patients from $20$ to $88$. The goal is to segment the tumor in each brain image by predicting its contour.

\begin{figure}[ht]
\centering
\centerline{\includegraphics[width=0.8\linewidth]{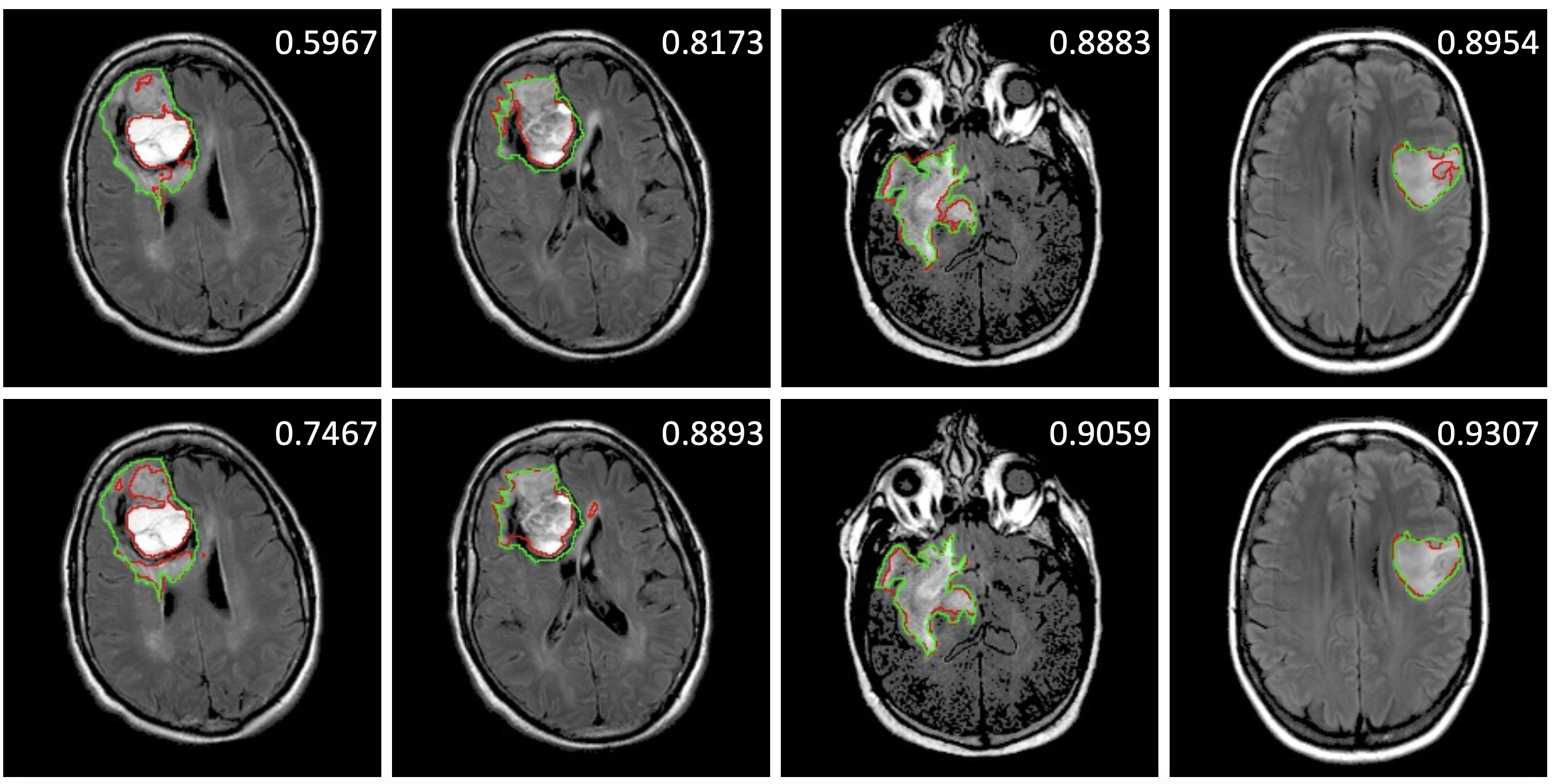}}
\caption{\textbf{Segmentation results by HNPs (bottom row) and U-Net (upper row)}. The ground truth contours are in green, and the predicted ones are in red. The numbers are DSC scores computed against the ground truth. HNPs can predict contours closer to the ground truth ones, indicating the advantages of exploring spatial context information for image segmentation.
}
\label{draw_brain}
\vspace{-2mm}
\end{figure}

We reformulate the segmentation task as a pixel-wise regression problem, where each pixel corresponds to a regression task to predict the probability of this pixel belonging to the tumor. In doing so, the spatial correlation and dependency among pixels are effectively modeled by capturing task-relatedness.
For the task $m$, we define $\Omega_m$ as a local region centered at the spatial position $m$, which provides the local context information.  In this case, the region centered at the pixel provides the local context information.
Each task incorporates the knowledge provided by related tasks into the context of the predictive function. This offers an effective way to model the long-range interdependence of pixels in one image. 
In the implementation, we use U-Net~\cite{ronneberger2015u} as the backbone and append our model to the last layer. 

\vspace{-2mm}
\paragraph{Results and Discussions.}
The proposed HNPs surpass the baseline U-Net by 0.8\%  (91.2\% v.s. 90.6\%) in terms of dice similarity coefficients (DSC) for the overall validation set. We provide the predicted contours by HNPs (bottom row) and the U-Net (upper row) in Figure~\ref{draw_brain}, where the green outline corresponds to the ground truth and the red to the segmentation output. This figure shows that HNPs predict contours closer to the ground truth. The results demonstrate the advantages of HNPs in exploring spatial-relatedness for medical segmentation.

%% file: main.bbl
\begin{thebibliography}{116}
\providecommand{\natexlab}[1]{#1}
\providecommand{\url}[1]{\texttt{#1}}
\expandafter\ifx\csname urlstyle\endcsname\relax
  \providecommand{\doi}[1]{doi: #1}\else
  \providecommand{\doi}{doi: \begingroup \urlstyle{rm}\Url}\fi

\bibitem[Goodfellow et~al.(2016)Goodfellow, Bengio, and
  Courville]{Goodfellow-et-al-2016}
Ian Goodfellow, Yoshua Bengio, and Aaron Courville.
\newblock \emph{Deep Learning}.
\newblock MIT Press, 2016.
\newblock \url{http://www.deeplearningbook.org}.

\bibitem[Xu et~al.(2021)Xu, Wang, Wang, and Zhu]{xu2021weak}
Shichao Xu, Lixu Wang, Yixuan Wang, and Qi~Zhu.
\newblock Weak adaptation learning: Addressing cross-domain data insufficiency
  with weak annotator.
\newblock In \emph{IEEE Conference on Computer Vision and Pattern Recognition},
  2021.

\bibitem[Johnson and Khoshgoftaar(2019)]{johnson2019survey}
Justin~M Johnson and Taghi~M Khoshgoftaar.
\newblock Survey on deep learning with class imbalance.
\newblock \emph{Journal of Big Data}, 6\penalty0 (1):\penalty0 1--54, 2019.

\bibitem[Vinyals et~al.(2016)Vinyals, Blundell, Lillicrap, Wierstra,
  et~al.]{vinyals2016matching}
Oriol Vinyals, Charles Blundell, Timothy Lillicrap, Daan Wierstra, et~al.
\newblock Matching networks for one shot learning.
\newblock In \emph{Advances in Neural Information Processing Systems}, 2016.

\bibitem[Finn et~al.(2017)Finn, Abbeel, and Levine]{finn2017model}
Chelsea Finn, Pieter Abbeel, and Sergey Levine.
\newblock Model-agnostic meta-learning for fast adaptation of deep networks.
\newblock In \emph{International Conference on Machine Learning}, 2017.

\bibitem[Snell et~al.(2017)Snell, Swersky, and Zemel]{snell2017prototypical}
Jake Snell, Kevin Swersky, and Richard Zemel.
\newblock Prototypical networks for few-shot learning.
\newblock In \emph{Advances in Neural Information Processing Systems}, 2017.

\bibitem[Requeima et~al.(2019)Requeima, Gordon, Bronskill, Nowozin, and
  Turner]{requeima2019fast}
James Requeima, Jonathan Gordon, John Bronskill, Sebastian Nowozin, and
  Richard~E Turner.
\newblock Fast and flexible multi-task classification using conditional neural
  adaptive processes.
\newblock In \emph{Advances in Neural Information Processing Systems}, 2019.

\bibitem[Caruana(1997)]{caruana1997multitask}
Rich Caruana.
\newblock Multitask learning.
\newblock \emph{Machine learning}, 28\penalty0 (1):\penalty0 41--75, 1997.

\bibitem[Zhang and Yang(2021)]{zhang2021survey}
Yu~Zhang and Qiang Yang.
\newblock A survey on multi-task learning.
\newblock \emph{IEEE Transactions on Knowledge and Data Engineering}, 2021.

\bibitem[Long et~al.(2017)Long, Cao, Wang, and Philip]{long2017learning}
Mingsheng Long, Zhangjie Cao, Jianmin Wang, and S~Yu Philip.
\newblock Learning multiple tasks with multilinear relationship networks.
\newblock In \emph{Advances in neural information processing systems}, 2017.

\bibitem[Shen et~al.(2021)Shen, Zhen, Worring, and Shao]{shen2021variational}
Jiayi Shen, Xiantong Zhen, Marcel Worring, and Ling Shao.
\newblock Variational multi-task learning with gumbel-softmax priors.
\newblock In \emph{Advances in Neural Information Processing Systems}, 2021.

\bibitem[Garnelo et~al.(2018{\natexlab{a}})Garnelo, Schwarz, Rosenbaum, Viola,
  Rezende, Eslami, and Teh]{garnelo2018neural}
Marta Garnelo, Jonathan Schwarz, Dan Rosenbaum, Fabio Viola, Danilo~J Rezende,
  SM~Eslami, and Yee~Whye Teh.
\newblock Neural processes.
\newblock \emph{arXiv preprint arXiv:1807.01622}, 2018{\natexlab{a}}.

\bibitem[Garnelo et~al.(2018{\natexlab{b}})Garnelo, Rosenbaum, Maddison,
  Ramalho, Saxton, Shanahan, Teh, Rezende, and Eslami]{garnelo2018conditional}
Marta Garnelo, Dan Rosenbaum, Christopher Maddison, Tiago Ramalho, David
  Saxton, Murray Shanahan, Yee~Whye Teh, Danilo Rezende, and SM~Ali Eslami.
\newblock Conditional neural processes.
\newblock In \emph{International Conference on Machine Learning},
  2018{\natexlab{b}}.

\bibitem[Bruinsma et~al.(2023)Bruinsma, Markou, Requeima, Foong, Andersson,
  Vaughan, Buonomo, Hosking, and Turner]{bruinsma2023autoregressive}
Wessel Bruinsma, Stratis Markou, James Requeima, Andrew Y.~K. Foong, Tom
  Andersson, Anna Vaughan, Anthony Buonomo, Scott Hosking, and Richard~E
  Turner.
\newblock Autoregressive conditional neural processes.
\newblock In \emph{International Conference on Learning Representations}, 2023.

\bibitem[Kim et~al.(2019)Kim, Mnih, Schwarz, Garnelo, Eslami, Rosenbaum,
  Vinyals, and Teh]{kim2019attentive}
Hyunjik Kim, Andriy Mnih, Jonathan Schwarz, Marta Garnelo, Ali Eslami, Dan
  Rosenbaum, Oriol Vinyals, and Yee~Whye Teh.
\newblock Attentive neural processes.
\newblock In \emph{International Conference on Learning Representations}, 2019.

\bibitem[Wang and van Hoof(2022)]{wang2022learning}
Qi~Wang and Herke van Hoof.
\newblock Learning expressive meta-representations with mixture of expert
  neural processes.
\newblock In \emph{Advances in Neural Information Processing Systems}, 2022.

\bibitem[Ravi and Larochelle(2017)]{ravi2017optimization}
Sachin Ravi and Hugo Larochelle.
\newblock Optimization as a model for few-shot learning.
\newblock In \emph{International Conference on Learning Representations}, 2017.

\bibitem[Wang and Van~Hoof(2020)]{wang2020doubly}
Qi~Wang and Herke Van~Hoof.
\newblock Doubly stochastic variational inference for neural processes with
  hierarchical latent variables.
\newblock In \emph{International Conference on Machine Learning}, 2020.

\bibitem[Kim et~al.(2021{\natexlab{a}})Kim, Cho, Lee, and Hong]{kim2021multi}
Donggyun Kim, Seongwoong Cho, Wonkwang Lee, and Seunghoon Hong.
\newblock Multi-task processes.
\newblock \emph{arXiv preprint arXiv:2110.14953}, 2021{\natexlab{a}}.

\bibitem[Guo et~al.(2023)Guo, Lan, Zhang, Lu, and Chen]{guo2023versatile}
Zongyu Guo, Cuiling Lan, Zhizheng Zhang, Yan Lu, and Zhibo Chen.
\newblock Versatile neural processes for learning implicit neural
  representations.
\newblock In \emph{International Conference on Learning Representations}, 2023.

\bibitem[Klenke(2013)]{klenke2013probability}
Achim Klenke.
\newblock \emph{Probability theory: a comprehensive course}.
\newblock Springer Science \& Business Media, 2013.

\bibitem[Kim et~al.(2021{\natexlab{b}})Kim, Son, and Kim]{kim2021vilt}
Wonjae Kim, Bokyung Son, and Ildoo Kim.
\newblock Vilt: Vision-and-language transformer without convolution or region
  supervision.
\newblock In \emph{International Conference on Machine Learning},
  2021{\natexlab{b}}.

\bibitem[S{\o}nderby et~al.(2016)S{\o}nderby, Raiko, Maal{\o}e, S{\o}nderby,
  and Winther]{sonderby2016ladder}
Casper~Kaae S{\o}nderby, Tapani Raiko, Lars Maal{\o}e, S{\o}ren~Kaae
  S{\o}nderby, and Ole Winther.
\newblock Ladder variational autoencoders.
\newblock In \emph{Advances in Neural Information Processing Systems}, 2016.

\bibitem[Kendall et~al.(2018)Kendall, Gal, and Cipolla]{kendall2018multi}
Alex Kendall, Yarin Gal, and Roberto Cipolla.
\newblock Multi-task learning using uncertainty to weigh losses for scene
  geometry and semantics.
\newblock In \emph{IEEE Conference on Computer Vision and Pattern Recognition},
  2018.

\bibitem[Liu et~al.(2019)Liu, Johns, and Davison]{liu2019end}
Shikun Liu, Edward Johns, and Andrew~J Davison.
\newblock End-to-end multi-task learning with attention.
\newblock In \emph{IEEE Conference on Computer Vision and Pattern Recognition},
  2019.

\bibitem[Misra et~al.(2016)Misra, Shrivastava, Gupta, and
  Hebert]{misra2016cross}
Ishan Misra, Abhinav Shrivastava, Abhinav Gupta, and Martial Hebert.
\newblock Cross-stitch networks for multi-task learning.
\newblock In \emph{IEEE Conference on Computer Vision and Pattern Recognition},
  2016.

\bibitem[Liu et~al.(2020)Liu, Li, Kuang, Xue, Chen, Yang, Liao, and
  Zhang]{liu2020towards}
Liyang Liu, Yi~Li, Zhanghui Kuang, Jing-Hao Xue, Yimin Chen, Wenming Yang,
  Qingmin Liao, and Wayne Zhang.
\newblock Towards impartial multi-task learning.
\newblock In \emph{International Conference on Learning Representations}, 2020.

\bibitem[Sener and Koltun(2018)]{sener2018multi}
Ozan Sener and Vladlen Koltun.
\newblock Multi-task learning as multi-objective optimization.
\newblock In \emph{Advances in Neural Information Processing Systems}, 2018.

\bibitem[Zamir et~al.(2018)Zamir, Sax, Shen, Guibas, Malik, and
  Savarese]{zamir2018taskonomy}
Amir~R Zamir, Alexander Sax, William Shen, Leonidas~J Guibas, Jitendra Malik,
  and Silvio Savarese.
\newblock Taskonomy: Disentangling task transfer learning.
\newblock In \emph{Proceedings of the IEEE conference on computer vision and
  pattern recognition}, 2018.

\bibitem[Bhattacharjee et~al.(2022)Bhattacharjee, Zhang, S{\"u}sstrunk, and
  Salzmann]{bhattacharjee2022mult}
Deblina Bhattacharjee, Tong Zhang, Sabine S{\"u}sstrunk, and Mathieu Salzmann.
\newblock Mult: an end-to-end multitask learning transformer.
\newblock In \emph{Proceedings of the IEEE/CVF Conference on Computer Vision
  and Pattern Recognition}, 2022.

\bibitem[Zhang et~al.(2021{\natexlab{a}})Zhang, Zhang, and
  Wang]{zhang2021multi}
Yi~Zhang, Yu~Zhang, and Wei Wang.
\newblock Multi-task learning via generalized tensor trace norm.
\newblock In \emph{Proceedings of the 27th ACM SIGKDD Conference on Knowledge
  Discovery \& Data Mining}, 2021{\natexlab{a}}.

\bibitem[Shen et~al.(2022)Shen, Xiao, Zhen, Snoek, and
  Worring]{shen2022association}
Jiayi Shen, Zehao Xiao, Xiantong Zhen, Cees~GM Snoek, and Marcel Worring.
\newblock Association graph learning for multi-task classification with
  category shifts.
\newblock \emph{arXiv preprint arXiv:2210.04637}, 2022.

\bibitem[Liang et~al.(2022)Liang, Meng, Xu, Chen, and Zhou]{liang2022scheduled}
Yunlong Liang, Fandong Meng, Jinan Xu, Yufeng Chen, and Jie Zhou.
\newblock Scheduled multi-task learning for neural chat translation.
\newblock \emph{arXiv preprint arXiv:2205.03766}, 2022.

\bibitem[Zhang et~al.(2022{\natexlab{a}})Zhang, Zhang, and
  Wang]{zhang2022learning}
Yi~Zhang, Yu~Zhang, and Wei Wang.
\newblock Learning linear and nonlinear low-rank structure in multi-task
  learning.
\newblock \emph{IEEE Transactions on Knowledge and Data Engineering},
  2022{\natexlab{a}}.

\bibitem[Bakker and Heskes(2003)]{bakker2003task}
Bart Bakker and Tom Heskes.
\newblock Task clustering and gating for bayesian multitask learning.
\newblock \emph{Journal of Machine Learning Research}, 2003.

\bibitem[Yu et~al.(2005)Yu, Tresp, and Schwaighofer]{yu2005learning}
Kai Yu, Volker Tresp, and Anton Schwaighofer.
\newblock Learning gaussian processes from multiple tasks.
\newblock In \emph{International Conference on Machine Learning}, 2005.

\bibitem[Titsias and L{\'a}zaro-Gredilla(2011)]{titsias2011spike}
Michalis~K Titsias and Miguel L{\'a}zaro-Gredilla.
\newblock Spike and slab variational inference for multi-task and multiple
  kernel learning.
\newblock In \emph{Advances in neural information processing systems}, 2011.

\bibitem[Lawrence and Platt(2004)]{lawrence2004learning}
Neil~D Lawrence and John~C Platt.
\newblock Learning to learn with the informative vector machine.
\newblock In \emph{International Conference on Machine Learning}, 2004.

\bibitem[Yousefi et~al.(2019)Yousefi, Smith, and {\'A}lvarez]{yousefi2019multi}
Fariba Yousefi, Michael~Thomas Smith, and Mauricio~A {\'A}lvarez.
\newblock Multi-task learning for aggregated data using gaussian processes.
\newblock \emph{arXiv preprint arXiv:1906.09412}, 2019.

\bibitem[Oyen and Lane(2012)]{oyen2012leveraging}
Diane Oyen and Terran Lane.
\newblock Leveraging domain knowledge in multitask bayesian network structure
  learning.
\newblock In \emph{Proceedings of the AAAI Conference on Artificial
  Intelligence}, 2012.

\bibitem[Qian et~al.(2020)Qian, Chen, Zhang, Wen, and Gechter]{qian2020multi}
Weizhu Qian, Bowei Chen, Yichao Zhang, Guanghui Wen, and Franck Gechter.
\newblock Multi-task variational information bottleneck.
\newblock \emph{arXiv preprint arXiv:2007.00339}, 2020.

\bibitem[Strezoski et~al.(2019{\natexlab{a}})Strezoski, Noord, and
  Worring]{strezoski2019learning}
Gjorgji Strezoski, Nanne~van Noord, and Marcel Worring.
\newblock Learning task relatedness in multi-task learning for images in
  context.
\newblock In \emph{International Conference on Multimedia Retrieval},
  2019{\natexlab{a}}.

\bibitem[Sun et~al.(2019)Sun, Panda, Feris, and Saenko]{sun2019adashare}
Ximeng Sun, Rameswar Panda, Rogerio Feris, and Kate Saenko.
\newblock Adashare: Learning what to share for efficient deep multi-task
  learning.
\newblock \emph{arXiv preprint arXiv:1911.12423}, 2019.

\bibitem[Strezoski et~al.(2019{\natexlab{b}})Strezoski, Noord, and
  Worring]{strezoski2019many}
Gjorgji Strezoski, Nanne~van Noord, and Marcel Worring.
\newblock Many task learning with task routing.
\newblock In \emph{IEEE International Conference on Computer Vision},
  2019{\natexlab{b}}.

\bibitem[Yu et~al.(2020)Yu, Kumar, Gupta, Levine, Hausman, and
  Finn]{yu2020gradient}
Tianhe Yu, Saurabh Kumar, Abhishek Gupta, Sergey Levine, Karol Hausman, and
  Chelsea Finn.
\newblock Gradient surgery for multi-task learning.
\newblock \emph{arXiv preprint arXiv:2001.06782}, 2020.

\bibitem[Liu et~al.(2021)Liu, Liu, Jin, Stone, and Liu]{liu2021conflict}
Bo~Liu, Xingchao Liu, Xiaojie Jin, Peter Stone, and Qiang Liu.
\newblock Conflict-averse gradient descent for multi-task learning.
\newblock In \emph{Advances in Neural Information Processing Systems}, 2021.

\bibitem[Fifty et~al.(2021)Fifty, Amid, Zhao, Yu, Anil, and
  Finn]{fifty2021efficiently}
Chris Fifty, Ehsan Amid, Zhe Zhao, Tianhe Yu, Rohan Anil, and Chelsea Finn.
\newblock Efficiently identifying task groupings for multi-task learning.
\newblock In \emph{Advances in Neural Information Processing Systems}, 2021.

\bibitem[Guo et~al.(2020{\natexlab{a}})Guo, Lee, and Ulbricht]{guo2020learning}
Pengsheng Guo, Chen-Yu Lee, and Daniel Ulbricht.
\newblock Learning to branch for multi-task learning.
\newblock In \emph{International Conference on Machine Learning},
  2020{\natexlab{a}}.

\bibitem[Thrun and Pratt(1998)]{thrun1998learning}
Sebastian Thrun and Lorien Pratt.
\newblock Learning to learn: Introduction and overview.
\newblock In \emph{Learning to learn}, pages 3--17. Springer, 1998.

\bibitem[Hospedales et~al.(2021)Hospedales, Antoniou, Micaelli, and
  Storkey]{hospedales2021meta}
Timothy~M Hospedales, Antreas Antoniou, Paul Micaelli, and Amos~J Storkey.
\newblock Meta-learning in neural networks: A survey.
\newblock \emph{IEEE Transactions on Pattern Analysis and Machine
  Intelligence}, 2021.

\bibitem[Allen et~al.(2019)Allen, Shelhamer, Shin, and
  Tenenbaum]{allen2019infinite}
Kelsey~R Allen, Evan Shelhamer, Hanul Shin, and Joshua~B Tenenbaum.
\newblock Infinite mixture prototypes for few-shot learning.
\newblock In \emph{International Conference on Machine Learning}, 2019.

\bibitem[Oreshkin et~al.(2018)Oreshkin, Rodr{\'\i}guez~L{\'o}pez, and
  Lacoste]{oreshkin2018tadam}
Boris Oreshkin, Pau Rodr{\'\i}guez~L{\'o}pez, and Alexandre Lacoste.
\newblock Tadam: Task dependent adaptive metric for improved few-shot learning.
\newblock In \emph{Advances in neural information processing systems}, 2018.

\bibitem[Yoon et~al.(2019)Yoon, Seo, and Moon]{yoon2019tapnet}
Sung~Whan Yoon, Jun Seo, and Jaekyun Moon.
\newblock Tapnet: Neural network augmented with task-adaptive projection for
  few-shot learning.
\newblock In \emph{International Conference on Machine Learning}, 2019.

\bibitem[Garcia and Bruna(2018)]{garcia2018few}
Victor Garcia and Joan Bruna.
\newblock Few-shot learning with graph neural networks.
\newblock In \emph{International Conference on Learning Representations}, 2018.

\bibitem[Cao et~al.(2019)Cao, Law, and Fidler]{cao2019theoretical}
Tianshi Cao, Marc Law, and Sanja Fidler.
\newblock A theoretical analysis of the number of shots in few-shot learning.
\newblock \emph{arXiv preprint arXiv:1909.11722}, 2019.

\bibitem[Triantafillou et~al.(2019)Triantafillou, Zhu, Dumoulin, Lamblin, Evci,
  Xu, Goroshin, Gelada, Swersky, Manzagol, and
  Larochelle]{triantafillou2019meta}
Eleni Triantafillou, Tyler Zhu, Vincent Dumoulin, Pascal Lamblin, Utku Evci,
  Kelvin Xu, Ross Goroshin, Carles Gelada, Kevin Swersky, Pierre-Antoine
  Manzagol, and Hugo Larochelle.
\newblock Meta-dataset: A dataset of datasets for learning to learn from few
  examples.
\newblock \emph{arXiv preprint arXiv:1903.03096}, 2019.

\bibitem[Sung et~al.(2018)Sung, Yang, Zhang, Xiang, Torr, and
  Hospedales]{sung2018learning}
Flood Sung, Yongxin Yang, Li~Zhang, Tao Xiang, Philip~HS Torr, and Timothy~M
  Hospedales.
\newblock Learning to compare: Relation network for few-shot learning.
\newblock In \emph{IEEE Conference on Computer Vision and Pattern Recognition},
  2018.

\bibitem[Nichol et~al.(2018)Nichol, Achiam, and Schulman]{nichol2018first}
Alex Nichol, Joshua Achiam, and John Schulman.
\newblock On first-order meta-learning algorithms.
\newblock \emph{arXiv preprint arXiv:1803.02999}, 2018.

\bibitem[Rusu et~al.(2019)Rusu, Rao, Sygnowski, Vinyals, Pascanu, Osindero, and
  Hadsell]{rusu2018meta}
Andrei~A Rusu, Dushyant Rao, Jakub Sygnowski, Oriol Vinyals, Razvan Pascanu,
  Simon Osindero, and Raia Hadsell.
\newblock Meta-learning with latent embedding optimization.
\newblock In \emph{International Conference on Learning Representations}, 2019.

\bibitem[Sung et~al.(2017)Sung, Zhang, Xiang, Hospedales, and
  Yang]{sung2017learning}
Flood Sung, Li~Zhang, Tao Xiang, Timothy Hospedales, and Yongxin Yang.
\newblock Learning to learn: Meta-critic networks for sample efficient
  learning.
\newblock \emph{arXiv preprint arXiv:1706.09529}, 2017.

\bibitem[Li et~al.(2017)Li, Zhou, Chen, and Li]{li2017meta}
Zhenguo Li, Fengwei Zhou, Fei Chen, and Hang Li.
\newblock Meta-sgd: Learning to learn quickly for few-shot learning.
\newblock \emph{arXiv preprint arXiv:1707.09835}, 2017.

\bibitem[Gordon et~al.(2019)Gordon, Bronskill, Bauer, Nowozin, and
  Turner]{gordon2018meta}
Jonathan Gordon, John Bronskill, Matthias Bauer, Sebastian Nowozin, and
  Richard~E Turner.
\newblock Meta-learning probabilistic inference for prediction.
\newblock In \emph{International Conference on Learning Representations}, 2019.

\bibitem[Edwards and Storkey(2016)]{edwards2016towards}
Harrison Edwards and Amos Storkey.
\newblock Towards a neural statistician.
\newblock \emph{arXiv preprint arXiv:1606.02185}, 2016.

\bibitem[Finn et~al.(2018)Finn, Xu, and Levine]{finn2018probabilistic}
Chelsea Finn, Kelvin Xu, and Sergey Levine.
\newblock Probabilistic model-agnostic meta-learning.
\newblock In \emph{Advances in Neural Information Processing Systems}, 2018.

\bibitem[S{\ae}mundsson et~al.(2018)S{\ae}mundsson, Hofmann, and
  Deisenroth]{saemundsson2018meta}
Steind{\'o}r S{\ae}mundsson, Katja Hofmann, and Marc~Peter Deisenroth.
\newblock Meta reinforcement learning with latent variable gaussian processes.
\newblock \emph{arXiv preprint arXiv:1803.07551}, 2018.

\bibitem[Baik et~al.(2021)Baik, Oh, Hong, and Lee]{baik2021learning}
Sungyong Baik, Junghoon Oh, Seokil Hong, and Kyoung~Mu Lee.
\newblock Learning to forget for meta-learning via task-and-layer-wise
  attenuation.
\newblock \emph{IEEE Transactions on Pattern Analysis and Machine
  Intelligence}, 2021.

\bibitem[Choi et~al.(2021)Choi, Choi, Baik, Kim, and Lee]{choi2021test}
Myungsub Choi, Janghoon Choi, Sungyong Baik, Tae~Hyun Kim, and Kyoung~Mu Lee.
\newblock Test-time adaptation for video frame interpolation via meta-learning.
\newblock \emph{IEEE Transactions on Pattern Analysis and Machine
  Intelligence}, 2021.

\bibitem[Raghu et~al.(2019)Raghu, Raghu, Bengio, and Vinyals]{raghu2019rapid}
Aniruddh Raghu, Maithra Raghu, Samy Bengio, and Oriol Vinyals.
\newblock Rapid learning or feature reuse? towards understanding the
  effectiveness of maml.
\newblock \emph{arXiv preprint arXiv:1909.09157}, 2019.

\bibitem[Wei et~al.(2021)Wei, Zhao, and Huang]{wei2021meta}
Ying Wei, Peilin Zhao, and Junzhou Huang.
\newblock Meta-learning hyperparameter performance prediction with neural
  processes.
\newblock In \emph{International Conference on Machine Learning}, 2021.

\bibitem[Markou et~al.(2022)Markou, Requeima, Bruinsma, Vaughan, and
  Turner]{markou2022practical}
Stratis Markou, James Requeima, Wessel~P Bruinsma, Anna Vaughan, and Richard~E
  Turner.
\newblock Practical conditional neural processes via tractable dependent
  predictions.
\newblock \emph{arXiv preprint arXiv:2203.08775}, 2022.

\bibitem[Ye and Yao(2022)]{ye2022contrastive}
Zesheng Ye and Lina Yao.
\newblock Contrastive conditional neural processes.
\newblock In \emph{IEEE Conference on Computer Vision and Pattern Recognition},
  2022.

\bibitem[Kim et~al.(2022)Kim, Go, and Yun]{kim2022neural}
Mingyu Kim, Kyeongryeol Go, and Se-Young Yun.
\newblock Neural processes with stochastic attention: Paying more attention to
  the context dataset.
\newblock \emph{arXiv preprint arXiv:2204.05449}, 2022.

\bibitem[Wang et~al.(2023)Wang, Federici, and van Hoof]{wang2023bridge}
Qi~Wang, Marco Federici, and Herke van Hoof.
\newblock Bridge the inference gaps of neural processes via expectation
  maximization.
\newblock In \emph{International Conference on Learning Representations}, 2023.

\bibitem[Song et~al.(2022)Song, Zheng, Cao, Yu, and Bian]{song2022efficient}
Xiaozhuang Song, Shun Zheng, Wei Cao, James Yu, and Jiang Bian.
\newblock Efficient and effective multi-task grouping via meta learning on task
  combinations.
\newblock \emph{Advances in Neural Information Processing Systems}, 2022.

\bibitem[Upadhyay et~al.(2023)Upadhyay, Chhipa, Phlypo, Saini, and
  Liwicki]{upadhyay2023multi}
Richa Upadhyay, Prakash~Chandra Chhipa, Ronald Phlypo, Rajkumar Saini, and
  Marcus Liwicki.
\newblock Multi-task meta learning: learn how to adapt to unseen tasks.
\newblock In \emph{2023 International Joint Conference on Neural Networks
  (IJCNN)}. IEEE, 2023.

\bibitem[Wang et~al.(2021{\natexlab{a}})Wang, Zhao, and Li]{wang2021bridging}
Haoxiang Wang, Han Zhao, and Bo~Li.
\newblock Bridging multi-task learning and meta-learning: Towards efficient
  training and effective adaptation.
\newblock In \emph{International Conference on Machine Learning},
  2021{\natexlab{a}}.

\bibitem[Tseng et~al.(2020)Tseng, Lee, Huang, and Yang]{tseng2020cross}
Hung-Yu Tseng, Hsin-Ying Lee, Jia-Bin Huang, and Ming-Hsuan Yang.
\newblock Cross-domain few-shot classification via learned feature-wise
  transformation.
\newblock In \emph{International Conference on Learning Representations}, 2020.

\bibitem[Chen et~al.(2019)Chen, Liu, Kira, Wang, and Huang]{chen2019closer}
Wei-Yu Chen, Yen-Cheng Liu, Zsolt Kira, Yu-Chiang~Frank Wang, and Jia-Bin
  Huang.
\newblock A closer look at few-shot classification.
\newblock \emph{arXiv preprint arXiv:1904.04232}, 2019.

\bibitem[Guo et~al.(2020{\natexlab{b}})Guo, Codella, Karlinsky, Codella, Smith,
  Saenko, Rosing, and Feris]{guo2020broader}
Yunhui Guo, Noel~C Codella, Leonid Karlinsky, James~V Codella, John~R Smith,
  Kate Saenko, Tajana Rosing, and Rogerio Feris.
\newblock A broader study of cross-domain few-shot learning.
\newblock In \emph{European Conference on Computer Vision}, 2020{\natexlab{b}}.

\bibitem[Du et~al.(2021)Du, Zhen, Shao, and Snoek]{du2020metanorm}
Yingjun Du, Xiantong Zhen, Ling Shao, and Cees G~M Snoek.
\newblock {MetaNorm}: Learning to normalize few-shot batches across domains.
\newblock In \emph{International Conference on Learning Representations}, 2021.

\bibitem[Das et~al.(2022)Das, Yun, and Porikli]{das2022confess}
Debasmit Das, Sungrack Yun, and Fatih Porikli.
\newblock Confess: A framework for single source cross-domain few-shot
  learning.
\newblock In \emph{International Conference on Learning Representations}, 2022.

\bibitem[Li et~al.(2022)Li, Liu, and Bilen]{li2022cross}
Wei-Hong Li, Xialei Liu, and Hakan Bilen.
\newblock Cross-domain few-shot learning with task-specific adapters.
\newblock In \emph{Proceedings of the IEEE/CVF Conference on Computer Vision
  and Pattern Recognition}, pages 7161--7170, 2022.

\bibitem[Vuorio et~al.(2019)Vuorio, Sun, Hu, and Lim]{vuorio2019multimodal}
Risto Vuorio, Shao-Hua Sun, Hexiang Hu, and Joseph~J Lim.
\newblock Multimodal model-agnostic meta-learning via task-aware modulation.
\newblock In \emph{Advances in neural information processing systems}, 2019.

\bibitem[Vuorio et~al.(2018)Vuorio, Sun, Hu, and Lim]{vuorio2018toward}
Risto Vuorio, Shao-Hua Sun, Hexiang Hu, and Joseph~J Lim.
\newblock Toward multimodal model-agnostic meta-learning.
\newblock In \emph{arXiv preprint arXiv:1812.07172}, 2018.

\bibitem[Yao et~al.(2019)Yao, Wei, Huang, and Li]{yao2019hierarchically}
Huaxiu Yao, Ying Wei, Junzhou Huang, and Zhenhui Li.
\newblock Hierarchically structured meta-learning.
\newblock In \emph{International Conference on Machine Learning}, 2019.

\bibitem[Abdollahzadeh et~al.(2021)Abdollahzadeh, Malekzadeh, and
  Cheung]{abdollahzadeh2021revisit}
Milad Abdollahzadeh, Touba Malekzadeh, and Ngai-Man~Man Cheung.
\newblock Revisit multimodal meta-learning through the lens of multi-task
  learning.
\newblock In \emph{Advances in Neural Information Processing Systems}, 2021.

\bibitem[Chen and Zhang(2021)]{chen2021hetmaml}
Jiayi Chen and Aidong Zhang.
\newblock Hetmaml: Task-heterogeneous model-agnostic meta-learning for few-shot
  learning across modalities.
\newblock In \emph{Proceedings of the 30th ACM International Conference on
  Information \& Knowledge Management}, 2021.

\bibitem[Xing et~al.(2019)Xing, Rostamzadeh, Oreshkin, and
  O~Pinheiro]{xing2019adaptive}
Chen Xing, Negar Rostamzadeh, Boris Oreshkin, and Pedro~O O~Pinheiro.
\newblock Adaptive cross-modal few-shot learning.
\newblock In \emph{Advances in Neural Information Processing Systems}, 2019.

\bibitem[Pahde et~al.(2018)Pahde, J{\"a}hnichen, Klein, and
  Nabi]{pahde2018cross}
Frederik Pahde, Patrick J{\"a}hnichen, Tassilo Klein, and Moin Nabi.
\newblock Cross-modal hallucination for few-shot fine-grained recognition.
\newblock \emph{arXiv preprint arXiv:1806.05147}, 2018.

\bibitem[Pahde et~al.(2021)Pahde, Puscas, Klein, and Nabi]{pahde2021multimodal}
Frederik Pahde, Mihai Puscas, Tassilo Klein, and Moin Nabi.
\newblock Multimodal prototypical networks for few-shot learning.
\newblock In \emph{Proceedings of the IEEE/CVF Winter Conference on
  Applications of Computer Vision}, 2021.

\bibitem[Venkateswara et~al.(2017)Venkateswara, Eusebio, Chakraborty, and
  Panchanathan]{venkateswara2017Deep}
Hemanth Venkateswara, Jose Eusebio, Shayok Chakraborty, and Sethuraman
  Panchanathan.
\newblock Deep hashing network for unsupervised domain adaptation.
\newblock In \emph{IEEE Conference on Computer Vision and Pattern Recognition},
  2017.

\bibitem[Peng et~al.(2019)Peng, Bai, Xia, Huang, Saenko, and
  Wang]{peng2019moment}
Xingchao Peng, Qinxun Bai, Xide Xia, Zijun Huang, Kate Saenko, and Bo~Wang.
\newblock Moment matching for multi-source domain adaptation.
\newblock In \emph{IEEE International Conference on Computer Vision}, 2019.

\bibitem[Gulrajani and Lopez-Paz(2020)]{gulrajani2020search}
Ishaan Gulrajani and David Lopez-Paz.
\newblock In search of lost domain generalization.
\newblock \emph{arXiv preprint arXiv:2007.01434}, 2020.

\bibitem[Wang et~al.(2021{\natexlab{b}})Wang, Miao, Zhen, and
  Qiu]{wang2021learning}
Ze~Wang, Zichen Miao, Xiantong Zhen, and Qiang Qiu.
\newblock Learning to learn dense gaussian processes for few-shot learning.
\newblock In \emph{Advances in Neural Information Processing Systems},
  2021{\natexlab{b}}.

\bibitem[Nguyen and Grover(2022)]{nguyen2022transformer}
Tung Nguyen and Aditya Grover.
\newblock Transformer neural processes: Uncertainty-aware meta learning via
  sequence modeling.
\newblock \emph{arXiv preprint arXiv:2207.04179}, 2022.

\bibitem[Upadhyay et~al.(2022)Upadhyay, Chhipa, Phlypo, Saini, and
  Liwicki]{upadhyay2022multi}
Richa Upadhyay, Prakash~Chandra Chhipa, Ronald Phlypo, Rajkumar Saini, and
  Marcus Liwicki.
\newblock Multi-task meta learning: learn how to adapt to unseen tasks.
\newblock \emph{arXiv preprint arXiv:2210.06989}, 2022.

\bibitem[Zhang et~al.(2022{\natexlab{b}})Zhang, Liao, Liu, Xu, and
  Zheng]{zhang2022leaving}
Qianqian Zhang, Xinru Liao, Quan Liu, Jian Xu, and Bo~Zheng.
\newblock Leaving no one behind: A multi-scenario multi-task meta learning
  approach for advertiser modeling.
\newblock In \emph{Proceedings of the Fifteenth ACM International Conference on
  Web Search and Data Mining}, 2022{\natexlab{b}}.

\bibitem[Cao et~al.(2021)Cao, You, and Leskovec]{cao2021relational}
Kaidi Cao, Jiaxuan You, and Jure Leskovec.
\newblock Relational multi-task learning: Modeling relations between data and
  tasks.
\newblock In \emph{International Conference on Learning Representations}, 2021.

\bibitem[Zhang et~al.(2021{\natexlab{b}})Zhang, Qian, Fang, and
  Xu]{zhang2021multi-modal}
Huaiwen Zhang, Shengsheng Qian, Quan Fang, and Changsheng Xu.
\newblock Multi-modal meta multi-task learning for social media rumor
  detection.
\newblock \emph{IEEE Transactions on Multimedia}, 24:\penalty0 1449--1459,
  2021{\natexlab{b}}.

\bibitem[Han et~al.(2022)Han, Yao, Zhao, Li, Shi, Wu, Chen, Qu, Zhao, Lan,
  et~al.]{han2022meta}
Chu Han, Huasheng Yao, Bingchao Zhao, Zhenhui Li, Zhenwei Shi, Lei Wu, Xin
  Chen, Jinrong Qu, Ke~Zhao, Rushi Lan, et~al.
\newblock Meta multi-task nuclei segmentation with fewer training samples.
\newblock \emph{Medical Image Analysis}, 2022.

\bibitem[Chen and Yeh(2021)]{chen2021mmtl}
Guan-Yuan Chen and Ya-Fen Yeh.
\newblock Mmtl: The meta multi-task learning for aspect category sentiment
  analysis.
\newblock In \emph{Proceedings of the 33rd Conference on Computational
  Linguistics and Speech Processing}, 2021.

\bibitem[Achille et~al.(2019)Achille, Lam, Tewari, Ravichandran, Maji, Fowlkes,
  Soatto, and Perona]{achille2019task2vec}
Alessandro Achille, Michael Lam, Rahul Tewari, Avinash Ravichandran, Subhransu
  Maji, Charless~C Fowlkes, Stefano Soatto, and Pietro Perona.
\newblock Task2vec: Task embedding for meta-learning.
\newblock In \emph{Proceedings of the IEEE/CVF international conference on
  computer vision}, 2019.

\bibitem[Nguyen et~al.(2021)Nguyen, Do, and Carneiro]{nguyen2021probabilistic}
Cuong~C Nguyen, Thanh-Toan Do, and Gustavo Carneiro.
\newblock Probabilistic task modelling for meta-learning.
\newblock In \emph{Uncertainty in Artificial Intelligence}, 2021.

\bibitem[Kingma and Welling(2013)]{kingma2013auto}
Diederik~P Kingma and Max Welling.
\newblock Auto-encoding variational bayes.
\newblock \emph{arXiv preprint arXiv:1312.6114}, 2013.

\bibitem[Kingma et~al.(2015)Kingma, Salimans, and
  Welling]{kingma2015variational}
Durk~P Kingma, Tim Salimans, and Max Welling.
\newblock Variational dropout and the local reparameterization trick.
\newblock In \emph{Advances in neural information processing systems}, 2015.

\bibitem[Simonyan and Zisserman(2014)]{simonyan2014very}
Karen Simonyan and Andrew Zisserman.
\newblock Very deep convolutional networks for large-scale image recognition.
\newblock \emph{arXiv preprint arXiv:1409.1556}, 2014.

\bibitem[He et~al.(2016)He, Zhang, Ren, and Sun]{he2016deep}
Kaiming He, Xiangyu Zhang, Shaoqing Ren, and Jian Sun.
\newblock Deep residual learning for image recognition.
\newblock In \emph{Proceedings of the IEEE conference on computer vision and
  pattern recognition}, 2016.

\bibitem[Kingma and Ba(2014)]{kingma2014adam}
Diederik~P Kingma and Jimmy Ba.
\newblock Adam: A method for stochastic optimization.
\newblock \emph{arXiv preprint arXiv:1412.6980}, 2014.

\bibitem[Yao et~al.(2021)Yao, Zhang, and Finn]{yao2021meta}
Huaxiu Yao, Linjun Zhang, and Chelsea Finn.
\newblock Meta-learning with fewer tasks through task interpolation.
\newblock \emph{arXiv preprint arXiv:2106.02695}, 2021.

\bibitem[Saenko et~al.(2010)Saenko, Kulis, Fritz, and
  Darrell]{saenko2010adapting}
Kate Saenko, Brian Kulis, Mario Fritz, and Trevor Darrell.
\newblock Adapting visual category models to new domains.
\newblock In \emph{European Conference on Computer Vision}, 2010.

\bibitem[Kulis et~al.(2011)Kulis, Saenko, and Darrell]{kulis2011you}
Brian Kulis, Kate Saenko, and Trevor Darrell.
\newblock What you saw is not what you get: Domain adaptation using asymmetric
  kernel transforms.
\newblock In \emph{CVPR 2011}. IEEE, 2011.

\bibitem[LeCun et~al.(1998)LeCun, Bottou, Bengio, and
  Haffner]{lecun1998gradient}
Yann LeCun, L{\'e}on Bottou, Yoshua Bengio, and Patrick Haffner.
\newblock Gradient-based learning applied to document recognition.
\newblock \emph{Proceedings of the IEEE}, 86\penalty0 (11):\penalty0
  2278--2324, 1998.

\bibitem[Gong et~al.(2012)Gong, Shi, Sha, and Grauman]{gong2012geodesic}
Boqing Gong, Yuan Shi, Fei Sha, and Kristen Grauman.
\newblock Geodesic flow kernel for unsupervised domain adaptation.
\newblock In \emph{IEEE Conference on Computer Vision and Pattern Recognition},
  2012.

\bibitem[Griffin et~al.(2007)Griffin, Holub, and Perona]{griffin2007caltech}
Gregory Griffin, Alex Holub, and Pietro Perona.
\newblock Caltech-256 object category dataset.
\newblock \emph{Dataset Report}, 2007.

\bibitem[Buda et~al.(2019)Buda, Saha, and Mazurowski]{buda2019association}
Mateusz Buda, Ashirbani Saha, and Maciej~A Mazurowski.
\newblock Association of genomic subtypes of lower-grade gliomas with shape
  features automatically extracted by a deep learning algorithm.
\newblock \emph{Computers in biology and medicine}, 109:\penalty0 218--225,
  2019.

\bibitem[Ronneberger et~al.(2015)Ronneberger, Fischer, and
  Brox]{ronneberger2015u}
Olaf Ronneberger, Philipp Fischer, and Thomas Brox.
\newblock U-net: Convolutional networks for biomedical image segmentation.
\newblock In \emph{International Conference on Medical Image Computing and
  Computer-Assisted Intervention}, 2015.

\end{thebibliography}
